\documentclass[preprint,12pt]{elsarticle}

\usepackage{amssymb}
\usepackage[utf8]{inputenc}
\usepackage[T1, T2A]{fontenc}
\usepackage[english]{babel}
\usepackage{graphicx}
\usepackage{caption}
\usepackage{subcaption}
\usepackage{geometry}
\usepackage{sidecap}
\usepackage{array}
\usepackage{float}
\usepackage[thinlines]{easytable}
\usepackage{fancyhdr}
\usepackage{amsmath,amsfonts,amssymb,amsthm}
\usepackage{mathtools}
\usepackage{commath}
\usepackage{adjustbox}
\usepackage{algorithm}
\usepackage[noend]{algpseudocode}
\usepackage{gensymb}
\usepackage{multirow}
\usepackage{color}
\usepackage{bm}
\usepackage{breakcites}
\usepackage{nomencl}
\makenomenclature
\usepackage{etoolbox}
\usepackage{hyperref}
\usepackage{booktabs}

\renewcommand\nomgroup[1]{%
	\item[\bfseries
	\ifstrequal{#1}{A}{Symbols}{%
		\ifstrequal{#1}{B}{Greek symbols}{%
			\ifstrequal{#1}{C}{Subscripts}{}}}%
	]}

\newcolumntype{L}{>{\centering\arraybackslash}m{3cm}}
\graphicspath{ {Slike/} }
\usepackage{hyperref}
\hypersetup{
	colorlinks=true, %set true if you want colored links
	linktoc=all,     %set to all if you want both sections and subsections linked
	linkcolor=black,  %choose some color if you want links to stand out
}

\newtheorem{theorem}{Theorem}
\newtheorem{lemma}{Lemma}

\def\BState{\State\hskip-\ALG@thistlm}

\journal{FAIR: Fair Adversarial Instance Re-weighting}

\begin{document}

\begin{frontmatter}



\title{FAIR: Fair Adversarial Instance Re-weighting}
\tnotetext[label1]{}


\author[FON]{Andrija Petrović\tnoteref{label1}}
\ead{aapetrovic@mas.bg.ac.rs}
\author[MATF]{Mladen Nikolić}
\author[FON]{Sandro Radovanović}
\author[FON]{Boris Delibašić}
\author[FON]{Miloš Jovanović}

\address[FON]{University of Belgrade - Faculty of Organizational Sciences, Jove Ilica 154, Belgrade, Serbia}
\address[MATF]{University of Belgrade - Faculty of Mathematics, Studentski Trg 16, Belgrade, Serbia}

\begin{abstract}
With growing awareness of societal impact of artificial intelligence, fairness has become an important aspect of machine learning algorithms. The issue is that human biases towards certain groups of population, defined by sensitive features like race and gender, are introduced to the training data through data collection and labeling. Two important directions of fairness ensuring research have focused on (i) instance weighting in order to decrease the impact of more biased instances and (ii) adversarial training in order to construct data representations informative of the target variable, but uninformative of the sensitive attributes. In this paper we propose a Fair Adversarial Instance Re-weighting (FAIR) method, which uses adversarial training to learn instance weighting function that ensures fair predictions. Merging the two paradigms, it inherits desirable  properties from both -- interpretability of reweighting and end-to-end trainability of adversarial training. We propose four different variants of the method and, among other things, demonstrate how the method can be cast in a fully probabilistic framework. Additionally, theoretical analysis of FAIR models' properties have been studied extensively. We compare FAIR models to 7 other related and state-of-the-art models and demonstrate that FAIR is able to achieve a better trade-off between accuracy and unfairness. To the best of our knowledge, this is the first model that merges reweighting and adversarial approaches by means of a weighting function that can provide interpretable information about fairness of individual instances.
\end{abstract}

% \section*{Research highlights}
%\begin{itemize}
%\item Hybrid model based on combination of two kinds of Gaussian Conditional Random Fields is proposed for traffic state estimation;
%\item Model is intended for prediction of spatially and temporally correlated outputs from sparse data;
%\item The proposed model is tested on two real-world large-scale networks;
%\item Advantages of proposed model are shown by comparison with state-of-the-art unstructured and structured predictors.
%\end{itemize}


\end{frontmatter}


\section{Introduction}
\label{Sec:Introduction}
Machine learning algorithms have lead to many recent breakthroughs in different complex tasks that cannot be solved satisfactorily by domain specific algorithms, such as face detection \cite{kumar2019face}, object detection \cite{voulodimos2018deep}, machine translation \cite{singh2017machine}, facial expression recognition
\cite{domadiya2019review}, sport prediction \cite{bunker2019machine}, etc. With this enormous success in practical applications and its growing presence in everyday life, social issues related to machine learning algorithms are becoming increasingly important. One of the most prominent issues is fairness of machine learning algorithms, related to discrimination and bias \cite{hajian2016algorithmic}.

It is well known that in many applications data reflects intended or unintended biases of humans whose actions generated data. Salary prediction \cite{innocenti2016mining}, credit risk prediction \cite{li2019credit}, medical prediction \cite{boyd1996relationship}, personnel planning and recruiting forecasting methods \cite{kim2016data}, are just some of the examples where data, collected from societal interactions, is biased with respect to age, gender, or race. Therefore, machine learning algorithms will extract and learn biases that are present in the data and these can have a strong discriminative impact towards disadvantaged groups.
Improving fairness of biased data and decision procedures based on that data is not only a problem of society, but also a problem of machine learning.  It is critical to guarantee that the prediction obtained by machine learning algorithms is based on appropriate information and that the outcomes are not biased towards certain groups of population defined by sensitive features like race and gender \cite{wang2019approaching}.

Current techniques for improving fairness fall into three different groups: preprocessing techniques \cite{kamiran2012decision,calmon2017optimized}, techniques based on optimization at training time \cite{zafar2019fairness,adel2019one,celis2019classification,kamishima2012fairness}, and post-processing based ones \cite{hardt2016equality, pleiss2017fairness}. State-of-the-art techniques for mitigating bias by preprocessing are based on instance reweighing \cite{kamiran2012data}, a technique that assigns weights to instances as means of controlling their influence on the model during training. The good side of such methods is that weights that the method assigns can be interpreted as indicators of instance fairness. The downside is that the preprocessing procedure is oblivious to
the properties of the downstream learning task, like loss function used, model architecture, etc. That may result in suboptimal weights with respect to that learning task.

Adversarial training has widely been used for finding Nash equilibrium in mini-max (zero-sum) games \cite{goodfellow2014generative, nouiehed2019solving, hsieh2019finding}. Recently, adversarial framework became popular in debiasing deep learning models by introducing two networks, one for predicting output labels and one for predicting sensitive attributes \cite{wadsworth2018achieving, madras2018learning, cevora2020fair, grari2020adversarial}. Both depend on the learnt feature space representation which allows fairly accurate prediction of the output label by the first network, while being maximally uninformative about the sensitive attributes, so that the second network has to fail in its task. While these methods allow for end-to-end training, they do not provide interpretable information on instance fairness, which is desirable.

In this paper, we propose Fair Adversarial Instance Re-weighting (FAIR) -- a novel model for mitigating bias in discriminative dataset by using an adversarial framework to learn an instance reweighing function instead of a new data representation as it is done in previous work. The weighting function can provide interpretable information on instance fairness. Also, FAIR does not perform weighting as preprocessing, but integrates it in the learning procedure so that the learning is performed end-to-end. FAIR consists of three neural networks: the first one is used for determining weights for each instance, the second one for predicting the sensitive attribute, and the third one for predicting the output label. FAIR comes in four variants differing in the weighting method. In the first method (FAIR-scalar), obtained scalar weights are used directly for weighting the log likelihood of corresponding instances, whereas in all other methods instance weights are modelled as random variables parametrized by the weighting network. In the second method (FAIR-Bernoulli) the weights are distributed according to Bernoulli distribution and during learning, score function is used to evaluate the expectation of the log likelihood. The other two methods rely on beta distribution, but they differ in evaluation of the expectation of the log likelihood -- the third one (FAIR-betaSF) uses score function and the fourth one (FAIR-betaREP) relies on reparametrization. Additionally, we discuss how to reduce the variance of FAIR-Bernoulli and FAIR-betaSF using baseline functions. We evaluated our models on four different real-world datasets and compared them to the state-of-the-art techniques. The results demonstrate that FAIR achieved the best results, with respect to fairness and classification performance. Furthermore, to the best of our knowledge, this is the first model that merges reweighting and adversarial approaches relying on a weighting function that can provide interpretable information about fairness of individual instances.

The remainder of the paper is structured as follows. In section~\ref{Sec:Related Work} the related work is reviewed. Adversarial models for debiasing datasets in probabilistic and non probabilistic framework are presented in section ~\ref{Sec:advnet}. The proposed FAIR algorithm with different variants is described in ~\ref{Sec:FAIR}.
Experimental setup and results on real-world applications are shown in sections ~\ref{Sec:exp-evaluation} and ~\ref{sec:ResultsDiscussion}, respectively. Final conclusions are given in section~\ref{Sec:Conclusion}.

\section{Related Work}
\label{Sec:Related Work}

\textbf{Notion of fairness}. In context of decision-making, (un)fairness has several distinct notions, one of the most prominent being {\em disparate impact} \cite{barocas2016big}. It represents a situation in which decisions ($\hat{\mathbf{y}}$) made by classifier are disproportional between instances with different values of sensitive attributes ($\mathbf{s}$). We use three measures of disparate impact. First metric used is \textit{absolute statistical parity difference}:
\begin{equation}
\mathbf{ASD} = |P(\hat{\mathbf{y}}=1|\mathbf{s}=0) - P(\hat{\mathbf{y}}=1|\mathbf{s}=1)|
\label{metric:asd}
\end{equation}
Low values of $\mathbf{ASD}$ mean that both groups have approximately the same probability of being labeled $1$ (e.g., bank loan granted) by the model. In such case, the classifier is said to have statistical parity.
Second metric we used is \textit{absolute equal opportunity difference}:
\begin{equation}
\mathbf{AEOD} = |TPR_{\mathbf{s}=0} - TPR_{\mathbf{s}=1}|
\label{metric:aeod}
\end{equation}
where $TPR$ represents true positive rate (recall) of the prediction model. Recall reflects opportunity, so this measure can be interpreted as a difference of opportunities between unprivileged and privileged group. Value of $\mathbf{AEOD}$ close to 0 is desirable. The third metric that we used is \textit{average odds difference}. Average odds difference can be formulated as:
\begin{equation}
\mathbf{AOD} = \frac{1}{2}(|FPR_{s=unpriv} - FPR_{s=priv}| + |TPR_{s=unpriv} - TPR_{s=priv}|)
\label{metric:aod}
\end{equation}
where $FPR$ represent false positive rate (probability of false alarm), and $TPR$ true positive rate (recall). Values of $\mathbf{AOD}$ close to zero are preferred.

Regarding fairness-aware machine learning algorithms, interested readers are referred to an extensive reviews presented in \cite{friedler2019comparative} and \cite{corbett2018measure}. We focus on two approaches relevant for our work.

\textbf{Instance Reweighting}. Instance reweighting has shown impressive results, although the idea is relatively simple \cite{feldman2015certifying, krasanakis2018adaptive, ren2018learning, shu2019meta, jiang2018mentornet}. As a preprocessing technique, it was traditionally used for the class imbalance problem by assigning larger weights to instances of a  lower cardinality class, so that the learning algorithm gives more importance to that class. This idea can be applied to the fairness problem as well -- it assigns lower importance to unfair examples or removes them from the learning process. More specifically, those examples will have a lower impact on the likelihood function one tries to optimize.
The simplest approach is to assign weights to instances so that sums of weights per value of a  sensitive feature is the same and all instances from a group have the same weight \cite{kamiran2012data}. That approach was improved in \cite{krasanakis2018adaptive} by utilizing adaptive sensitive reweighting procedure. One can use variational fair auto encoder with Maximum Mean Discrepancy \cite{louizos2015variational} which calculates distances between distributions using kernels. It is worth noticing that instance reweighting has shown to have lower disparate impact \cite{feldman2015certifying} compared to not applying any instance weighting strategy. An advantage of this class of methods is that the weights can be interpreted as indicators of individual instance fairness. However, the training is not end-to-end. To apply some instance reweighting strategy one needs to perform a two step procedure -- first to obtain instance weights, and then to use the weights by the learning algorithm. This is a drawback of this approach since the weighting procedure is oblivious of the model representation and learning algorithm and therefore might choose suboptimal weights for them.

\textbf{Adversarial training}. Adversarial training provides a framework for mitigating biases by learning new data representation from which it is possible to predict the target variable, but not possible to predict the sensitive attribute. This approach creates a trade-off between two goal functions and therefore reaches Nash equilibrium \cite{goodfellow2014generative}.
Adversarial training for fairness was first presented in \cite{zhang2018mitigating}. Similar model was applied to recidivism prediction in order to remove racial bias \cite{wadsworth2018achieving}.
An important approach of such kind is Fair Adversarial Discriminative model (FAD) \cite{adel2019one}. Moreover, theoretical analysis of the relationship between the label classifier performance and the adversary’s ability to predict the sensitive attribute value is provided. Also, in the same paper, a variation of the adversarial learning procedure is developed to increase diversity among elements of each mini-batch of the gradient descent training, in order to achieve a representation that does not suffer from mode collapse.
Another theoretical analysis of solving fairness problem via adversarial approach is presented in \cite{madras2018learning}.
Another adversarial approach focuses on learning to select non-sensitive features on per instance basis \cite{wang2019approaching}. The adversarial approach is employed to minimize the correlation between selected features and sensitive information. While adversarial approach enables end-to-end training it does not provide any interpretable information on the individual fairness of instances.

\vspace{2mm}
Our approach tries to keep the best from both worlds. It provides interpretable information on instance fairness like reweighting approach and allows end-to-end training like adversarial approach.

\section{Fairness via Adversarial Network}
\label{Sec:advnet}
The dataset given by $D= \{(\mathbf{x}_i, \mathbf{y}_i, \mathbf{s}_i) \}_{i=1}^N$, consists of input features $\mathbf{x}$, the true label (or the target variable) $\mathbf{y}$ and sensitive features $\mathbf{s}$. It is generated by joint true underlying distribution $D \sim P(\mathbf{x},\mathbf{y},\mathbf{s})$. Generally, an unfair discriminative model predicts the label $\mathbf{\hat{y}}$ based on both input $\mathbf{x}$ and sensitive  features $\mathbf{s}$, which can lead to bias. A naive approach to ensuring fairness would be to eliminate sensitive features $\mathbf{s}$ from the dataset. However, the information contained in the sensitive features can often be approximated from other input features $\mathbf{x}$. For example, location of residence correlates with race, although it is not obviously sensitive itself. In this section, we discuss already mentioned fair adversarial discriminative model (FAD) in more detail, as a relevant baseline.
Also, we propose our probabilistic formulation of this model based on normalizing flows. This is not the main contribution of our paper, though. Instead, we use it as another reasonable baseline for our other probabilistic models.

\subsection{FAD Model}

The architecture of the FAD model consists of one shared network and two task specific networks. The goal of the shared network is to map input features $\mathbf{x}$ to their new representation $\mathbf{z}=f_\theta(\mathbf{x})$, so that the obtained representation $\mathbf{z}$ is uninformative of sensitive features $\mathbf{s}$, but includes information needed to predict label $\mathbf{y}$. The first task specific network is a predictor $g_\phi(\mathbf{z})$ of the output label $\mathbf{y}$, whereas the second task specific network $h_\psi(\mathbf{z})$ estimates sensitive features $\mathbf{s}$. Since the sensitive information may also be important for estimating labels, there is a trade-off between model fairness and the accuracy of the prediction, related to the mapping from input feature space $\mathbf{x}$ to representation space $\mathbf{z}$.
Fairness in the FAD method is achieved through adversarial learning of the mapping $g_\theta(\mathbf{x})$ and a classifier $h_\psi(\mathbf{z})$, while learning the predictor $g_\phi(\mathbf{z})$. This ensures that the accuracy is not fully sacrificed for fairness. In other words, neural networks $f_\theta(\mathbf{x})$ and $g_\phi(\mathbf{z})$ play a minimax game with the classifier $h_\psi(\mathbf{z})$. We denote probability functions modelled by these networks as $P_\phi(\mathbf{y}|\mathbf{x})$ and $P_\psi(\mathbf{s}|\mathbf{x})$. Formally, the adversarial problem of FAD model is:
$$\min_{\theta,\phi}\max_{\psi} \mathbb{E}_{\substack{\mathbf{x},\mathbf{y},\mathbf{s} \sim P(\mathbf{x},\mathbf{y},\mathbf{s})\\\mathbf{z}=f_\theta(\mathbf{x})}} \left[\alpha\log P_{\psi}(\mathbf{s}|\mathbf{z}) - \log P_{\phi}(\mathbf{y}|\mathbf{z})\right]$$
%= \sum_{i=1}^N \mathbf{s}_i\cdot log\left(\phi(g(\mathbf{x}_i))\right) + \left(1-\mathbf{s}_i\right)\cdot\left(1-\phi(g(\mathbf{x}_i))\right)
where $\alpha$ is a hyper-parameter for tuning the trade-off between the model fairness and accuracy. Increased value of $\alpha$ influences model to be more focused on fairness and, consequently, representation $\mathbf{z}$ will be less informative of sensitive features $\mathbf{s}$, but also, to some extent, of true label $\mathbf{y}$.

\subsection{Probabilistic framework with normalizing flows}
\label{Sec:advenet-nf}
In case of the FAD method, the representation $\mathbf{z}=f_\theta(\mathbf{x})$ is an output of a neural network. We propose a fully probabilistic method (FAD-prob) based on the FAD method by considering the representation $\mathbf{z}$ as a random latent variable and modeling its distribution. Representing the hidden space as a random variable has several advantages. The main one is related to the possibility to marginalize over latent variable space and obtain better predictive performance \cite{tan2010social}. Moreover, it can provide a possibility to predict structured outputs representing sensitive features and labels \cite{koller2009probabilistic}.

The conditional probability distributions of sensitive features $\mathbf{s}$ and of output labels $\mathbf{y}$,  given inputs $\mathbf{x}$ can be obtained by marginalization of joint distributions $P(\mathbf{s},\mathbf{z})$ and $P(\mathbf{y},\mathbf{z})$, respectively as
$$ P(\mathbf{y}|\mathbf{x}) = \int_\mathbf{z} P(\mathbf{y}|\mathbf{z}) P(\mathbf{z}|\mathbf{x}) d\mathbf{z} = \mathbb{E}_{\mathbf{z} \sim P(\mathbf{z}|\mathbf{x})} [P(\mathbf{y}|\mathbf{z})]$$
$$ P(\mathbf{s}|\mathbf{x}) = \int_\mathbf{z} P(\mathbf{s}|\mathbf{z}) P(\mathbf{z}|\mathbf{x}) d\mathbf{z} = \mathbb{E}_{\mathbf{z} \sim P(\mathbf{z}|\mathbf{x})} [P(\mathbf{s}|\mathbf{z})]$$
Marginalization can be performed using reparametrization trick and normalizing flows \cite{kingma2019introduction}. The reparametrization of the latent variable can be performed as $\mathbf{z} = f(\mathbf{\mu(\mathbf{x}})+ L(\mathbf{x})\cdot \mathbf{\epsilon})$
where $f$ is a nonlinear mapping of variable $\mathbf{z}$ obtained by assuming normal distribution with mean $\mu(\mathbf{x})$ and covariance matrix $\Sigma(\mathbf{x})$, which can be factorized as $L^T(\mathbf{x})L(\mathbf{x})$ by Cholesky decomposition.

Based on this, the overall adversarial objective function of FAD-prob model is:
$$\min_{\theta,\phi}\max_{\psi}\mathbb{E}_{\substack{\mathbf{x},\mathbf{y}, \mathbf{s} \sim P(\mathbf{x},\mathbf{y},\mathbf{s})\\\mathbf{z} \sim P_\theta(\mathbf{z}|\mathbf{x})}} \left[\alpha \log P_{\psi}(\mathbf{s}|\mathbf{z})  - \log P_{\phi}(\mathbf{y}|\mathbf{z})\right]$$

\section{Fair Adversarial Instance Re-weighting - FAIR}
\label{Sec:FAIR}
Unfairness which AI models learn is introduced through data instances containing unfair decisions. Therefore, we strive to recognize if a particular instance in a dataset is unfair. The main principle of FAIR is to reweight log likelihood of each instance, according to the trade-off between fairness and prediction performance, in order to obtain a fair and useful predictor of the target variable.

FAIR consists of three neural networks: the weighting network $f_\theta(\mathbf{x})$, the predictor network $g_\phi(\mathbf{x})$, and the sensitive network $h_\psi(\mathbf{x})$.
For an instance $\mathbf{x}$ the weighting network outputs the weight of that instance $w_\mathbf{x}\in[0,1]$, while the predictor network and the sensitive network output predictions of the output labels $\mathbf{y}$ and the sensitive features $\mathbf{s}$, respectively. In order to incorporate the fairness objective, FAIR weights log likelihood of instances, so that the ones that are strongly informative of the sensitive features, but not of the target variable are assigned low weights and the ones that are informative of the target variable, but not of the sensitive attributes are assigned high weights.
The weighting network is not used during inference, but can be helpful for assessing new instances.

Based on different weighting techniques, we present four different FAIR weighting methods. The first one, FAIR-scalar is based on non-probabilistic weighting framework, whereas FAIR-Bernoulli, FAIR-betaSF, and FAIR-betaREP are based on probabilistic framework. The graphical representation of FAIR with different weighting methods are given in Fig.~\ref{fig:Fig1}.

\begin{figure*}[t!]
	\vskip 0.2in
	\center
	\includegraphics[angle=0, width=1\textwidth, height = 2.7in]{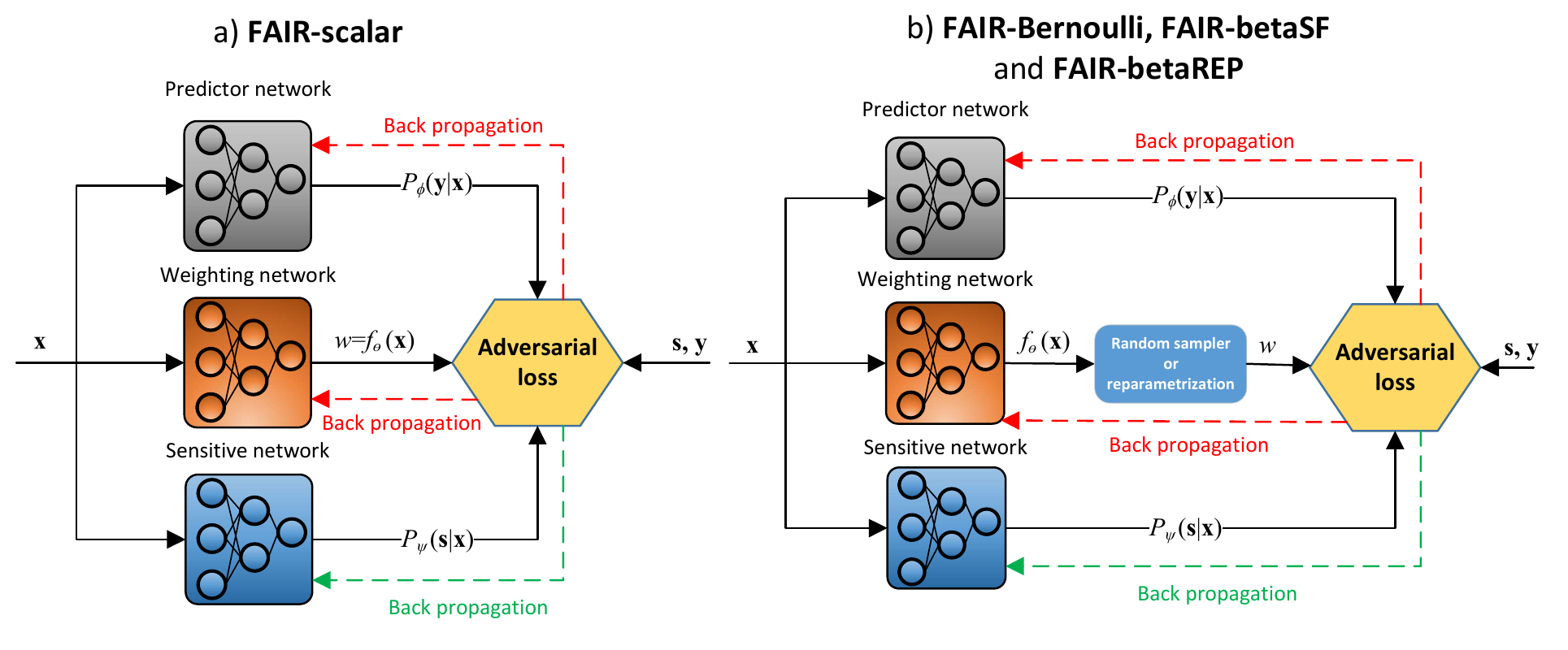}
	\captionsetup{justification=centering}
	\caption{Graphical representations of FAIR with probabalistic and non-probabalistic frameworks}
	\label{fig:Fig1}
	\vskip -0.2in
\end{figure*}

\subsection{FAIR -- non-probablistic framework}
Assume that each instance $\mathbf{x}$ is assigned a scalar weight $f_\theta(\mathbf{x})\in[0,1]$ by a weighting network. Then, FAIR-scalar adversarial problem is given by:
\begin{equation}
\label{Eq:Loss-scalar}
(\theta^*,\phi^*,\psi^*) = \arg\min_{\theta,\phi}\max_{\psi}\mathbb{E}_{\substack{\mathbf{x},\mathbf{y},\mathbf{s} \sim P(\mathbf{x},\mathbf{y},\mathbf{s})\\w=f_\theta(\mathbf{x})}} [w \cdot(\alpha\log P_{\psi}(\mathbf{s}|\mathbf{x})
- \log P_{\phi}(\mathbf{y}|\mathbf{x}))]
\end{equation}
Similarly to FAD model, the hyperparameter $\alpha$ controls the trade-off between fairness and predictive performance of the predictor network, but this trade-off will be given further theoretical analysis.

\subsection{FAIR -- probabilistic framework}
\label{Sec:FAD-pf}
In the the case of FAIR with probabilistic approach to weighting, it is assumed that  weights of instances are random variables. In contrast to FAIR-scalar, in probabilistic framework, output of the weighting network $f_\theta$ models a probability distribution of instance weights: $P(w_\mathbf{x}|\mathbf{x})$. Consequently, we can use different probability distribution models. We consider Bernoulli (FAIR-Bernoulli) and beta distribution (FAIR-betaSF and FAIR-betaREP).

FAIR-Bernoulli assumes that log likelihoods of instances, with respect to sensitive features $\log P_{\psi}(\mathbf{s}|\mathbf{x})$ and labels $\log P_{\phi}(\mathbf{y}|\mathbf{x})$ are weighted by integers $w_\mathbf{x} \in \{0,1\}$ such that it holds $P_\theta(w_\mathbf{x}=1|\mathbf{x})=f_\theta(x)$, meaning that the conditional probability of weights is a Bernoulli distribution ${\cal B}(f_\theta(\mathbf{x}))$. The FAIR-Bernoulli adversarial loss ${\cal L}^{\cal B}_{\alpha}(\theta, \phi, \psi)$ is given by:
\begin{gather}
\label{Eq:Loss_prob}
\mathbb{E}_{\substack{\mathbf{x},\mathbf{y}, \mathbf{s} \sim P\left(\mathbf{x},\mathbf{y}, \mathbf{s}\right)\\w \sim P_\theta(w|\mathbf{x})}}  \left[ w \cdot(\alpha\log P_{\psi}\left(\mathbf{s}|\mathbf{x}\right)
-  \log P_{\phi}\left(\mathbf{y}|\mathbf{x}\right))\right]
\end{gather}
and the corresponding adversarial problem is $(\theta^*,\phi^*,\psi^*) = \arg\min_{\theta,\phi}\max_{\psi}{\cal L}^{\cal B}_\alpha(\theta,\phi,\psi)$
where the superscript ${\cal B}$ emphasizes Bernoulli assumption.

In order to optimize the loss, gradients with respect to $\theta$, $\phi$, and $\psi$ need to be computed. Gradients with respect to $\phi$ and $\psi$ are computed by standard backpropagation. However, the gradient with respect to $\theta$ is trickier since $\theta$ defines the distribution of $w$ over which the expectation is taken. Therefore, we derive the gradient of the adversarial loss $\nabla_{\theta}{\cal L}_{\alpha}(\theta, \phi, \psi)$ for FAIR-Bernoulli and FAIR-betaSF as follows:
$$\nabla_{\theta}{\cal L}_{\alpha}(\theta, \phi, \psi) = \nabla_{\theta} \mathbb{E}_{\substack{\mathbf{x},\mathbf{y}, \mathbf{s} \sim P(\mathbf{x},\mathbf{y}, \mathbf{s})\\w \sim P_\theta(w|\mathbf{x})}} \bigg[ w \cdot ( \alpha \log P_{\psi}(\mathbf{s}|\mathbf{x}) - \log P_{\phi}(\mathbf{y}|\mathbf{x}))\bigg]$$
The gradient operator $\nabla_{\theta}$ can be propagated through the expectation as:
\begin{align*}
\begin{split}
\nabla_{\theta}{\cal L}_{\alpha}(\theta, \phi, \psi) =& \mathbb{E}_{\mathbf{y},\mathbf{x},\mathbf{s}} \bigg[ \int_w \nabla_{\theta}  P_\theta(w|\mathbf{x}) \cdot w \cdot ( \alpha \log P_{\psi}(\mathbf{s}|\mathbf{x}) - \log P_{\phi}\left(\mathbf{y}|\mathbf{x}\right))dw \bigg]
\end{split}
\end{align*}
Gradient of the distribution $P_\theta(w|\mathbf{x})$ can be transformed as:
\begin{align*}
\begin{split}
\nabla_{\theta}  P_\theta(w|\mathbf{x}) &= P_\theta(w|\mathbf{x}) \cdot \frac{\nabla_{\theta}P_\theta(w|\mathbf{x}) }{ P_\theta(w|\mathbf{x})}
\\
&= P_\theta(w|\mathbf{x}) \cdot \nabla_{\theta} \log P_\theta(w|\mathbf{x})
\end{split}
\end{align*}

Following this transformation, the final form of the gradient of the loss with respect to $\theta$ can be represented as:
\begin{align}
\label{Eq:Gradient}
\begin{split}
\mathbb{E}_{\substack{\mathbf{x}, \mathbf{s}, \mathbf{y}\sim P(\mathbf{x},\mathbf{s}, \mathbf{y})\\w \sim P_\theta(w|\mathbf{x})}} \Big[w\cdot \nabla_{\theta}  \log P_\theta(w|\mathbf{x}) \cdot (\alpha\log P_\psi(\mathbf{s}|\mathbf{x}) - \log P_\phi(\mathbf{y}|\mathbf{x})) \Big]
\end{split}
\end{align}
which is a suitable form as it allows the use of the stochastic gradient descent.

Next, we assume that weights $w_\mathbf{x}$ are random variables distributed according to the beta distribution which, in contrast to the case of FAIR-Bernoulli, takes any value from the interval $[0,1]$. The outputs of the weighting network are the parameters $\alpha_\mathbf{x}$ and $\beta_\mathbf{x}$ of the beta distribution. The adversarial loss as defined by Eq.~\ref{Eq:Loss_prob}, but with beta distribution assumed instead of Bernoulli. We denote corresponding loss by ${\cal L}^\beta_\alpha(\theta,\phi,\psi)$ where $\beta$ in the superscript emphasizes the assumed distribution. In optimization, the gradient $\nabla_{\theta}{\cal L}^\beta_{\alpha}(\theta, \phi, \psi)$
can be evaluated either by using score function as in Eq.~\ref{Eq:Gradient} or by the reparametrization trick of beta distribution as shown in \cite{shah2015empirical}. These two approaches we name FAIR-betaSF and FAIR-betaREP respectively.

Pseudocode of probabalistic FAIR with score function (FAIR-Bernoulli and FAIR-betaSF) is presented in Algortihm~\ref{alg:FAIR-sf}. FAIR losses are defined in terms of expectations. However, with finite samples, expectation is always approximated by sample mean, which we use in the algorithm.

\begin{algorithm}
	\caption{Probabilistic FAIR with score function}
	\label{alg:FAIR-sf}
	\begin{algorithmic}
		\State {\bfseries Input:} learning rates $\gamma_\theta, \gamma_\phi, \gamma_\psi$, dataset $D$,  hyperparameter $\alpha$, probabilistic model ${\cal P}$ of instance weights, number of iterations $M$
		\State {\bfseries Output:} parameters $\theta,\phi,\psi$
		\vspace{2mm}
		\State Initialize $\theta$, $\phi$, $\psi$
		\For{i = 1 to M}
		\State Sample a mini-batch $B\subseteq D$
		\State Sample $w_\mathbf{x}\sim {\cal P}(f_\theta(\mathbf{x}))$ for each $\mathbf{x}$ in $B$
		\State $d_{\theta} \leftarrow \gamma_\theta\frac{1}{|B|} \sum_{(\mathbf{x},\mathbf{y},\mathbf{s})\in B}\left[w_\mathbf{x}\nabla_{\theta}\log P_\theta(w_\mathbf{x}|\mathbf{x})\cdot\right.$
		
		\hspace{3cm}$\left.(\alpha\log P_\psi(\mathbf{s}|\mathbf{x})-\log P_\phi(\mathbf{y}|\mathbf{x}))\right]$
		\State $d_{\phi} \leftarrow \gamma_\phi\nabla_{\phi}{\cal L}^{\cal P}_\alpha(\theta,\phi,\psi,B)$
		\State $d_{\psi} \leftarrow -\gamma_\psi\nabla_{\psi}{\cal L}^{\cal P}_\alpha(\theta,\phi,\psi,B)$
		\State $(\theta,\phi,\psi) \leftarrow (\theta,\phi,\psi) - (d_{\theta}, d_{\phi}, d_{\psi})$
		\EndFor
	\end{algorithmic}
\end{algorithm}

\subsubsection{FAIR-Bernoulli and FAIR-betaSF with baselines}
\label{app:baselines}

Baseline functions are a commonly used tool to reduce the variance of the estimate of the gradient in reinforcement learning algorithms. It is shown that introducing a baseline in loss function does not introduce additional bias into the model \cite{sutton2018reinforcement}. Here we explain how these techniques can be incorporated in FAIR models. These modifications are not included in experimental evaluation since the technique is already known and our main goal is to compare basic FAIR variants against existing baselines, but we still derive algorithms with baselines as they might be of practical importance.

Already discussed adversarial loss is augmented by adding another term -- the baseline loss:
$${\cal L}_{\alpha}(\mu) = \mathrm{Var} \left[w\cdot \nabla_{\theta_g}\log P_g(w|\mathbf{x}) \cdot\left( \alpha \log P_\psi(\mathbf{s}|\mathbf{x}) - \log P_\phi(\mathbf{y}|\mathbf{x}) \right)
     - b_\mu(\mathbf{x}) \right]$$
The baseline loss includes the gradient of the adversarial loss, since its purpose is to reduce the variance of the estimate of that gradient.
Keeping in mind that the variance can be represented as $\mathrm{Var}[x] = \mathbb{E}[x^2] - \mathbb{E}[x]^2$ (where squaring of a vector $v$ means $v^Tv$), the baseline loss can be simplified due to the fact that it holds
$\mathbb{E}_{P_\theta(w|\mathbf{x})}[\nabla_{\theta}\log P_\theta(w|\mathbf{x}) b_\mu(\mathbf{x})]=0$ \cite{sutton2018reinforcement}:
\begin{gather*}
	\begin{split}
	&{\cal L}_{\alpha}(\mu) = \mathbb{E} \left[\left(w\cdot\nabla_{\theta}\log P_\theta(w|\mathbf{x})\cdot (\alpha \log P_\psi(\mathbf{s}|\mathbf{x}) - \log P_\phi(\mathbf{y}|\mathbf{x}))  - b_\mu(\mathbf{x})\right)^2 \right]
	\end{split}
\end{gather*}
where the expectation is with respect to $(\mathbf{x},\mathbf{y},\mathbf{s})\sim P(\mathbf{x},\mathbf{y},\mathbf{s})$ and $w\sim P(w|\mathbf{x})$.
Furthermore, we assumed the independence among the values involved in the expectation, and thus the expectation can be represented as:
\begin{equation*}
	\begin{split}
	{\cal L}_{\alpha}(\mu) = \mathbb{E} \left[\left(\nabla_{\theta}\log P_\theta(w|\mathbf{x})\right)^2 \right] \cdot \mathbb{E} \left[ \left(w \cdot (\alpha \log P_\psi(\mathbf{s}|\mathbf{x}) - \log P_\phi(\mathbf{y}|\mathbf{x})) - b_\mu(\mathbf{x})\right)^2 \right]
	\end{split}
\end{equation*}

Considering that the first factor is constant with respect to $b_\mu$ it can be omitted, so that the final form of the loss ${\cal L}_{\alpha}(\mu)$  is:
\begin{equation*}
\begin{split}
\mathbb{E} \left[\left(w \cdot (\alpha \log P_\psi(\mathbf{s}|\mathbf{x}) - \log P_\phi(\mathbf{y}|\mathbf{x})) - b_\mu(\mathbf{x})\right)^2\right]
\end{split}
\end{equation*}

Then, the gradient $\nabla_\mu{\cal L}_\alpha(\mu)$ is:
\begin{equation*}
\begin{split}
-\mathbb{E} \left[w \cdot \nabla_\mu b_\mu(\mathbf{x})\cdot\left(\alpha \log P_\psi(\mathbf{s}|\mathbf{x}) - \log P_\phi(\mathbf{y}|\mathbf{x}) - b_\mu(\mathbf{x})\right)\right]
\end{split}
\end{equation*}

In Fig.~\ref{fig:Fig2} graphical representation of FAIR-betaSF with baseline is shown. The pseudo-code is presented in Algorithm~\ref{alg:beta-base}.

\begin{figure*}[h!]
	\center
	\includegraphics[width=1\textwidth]{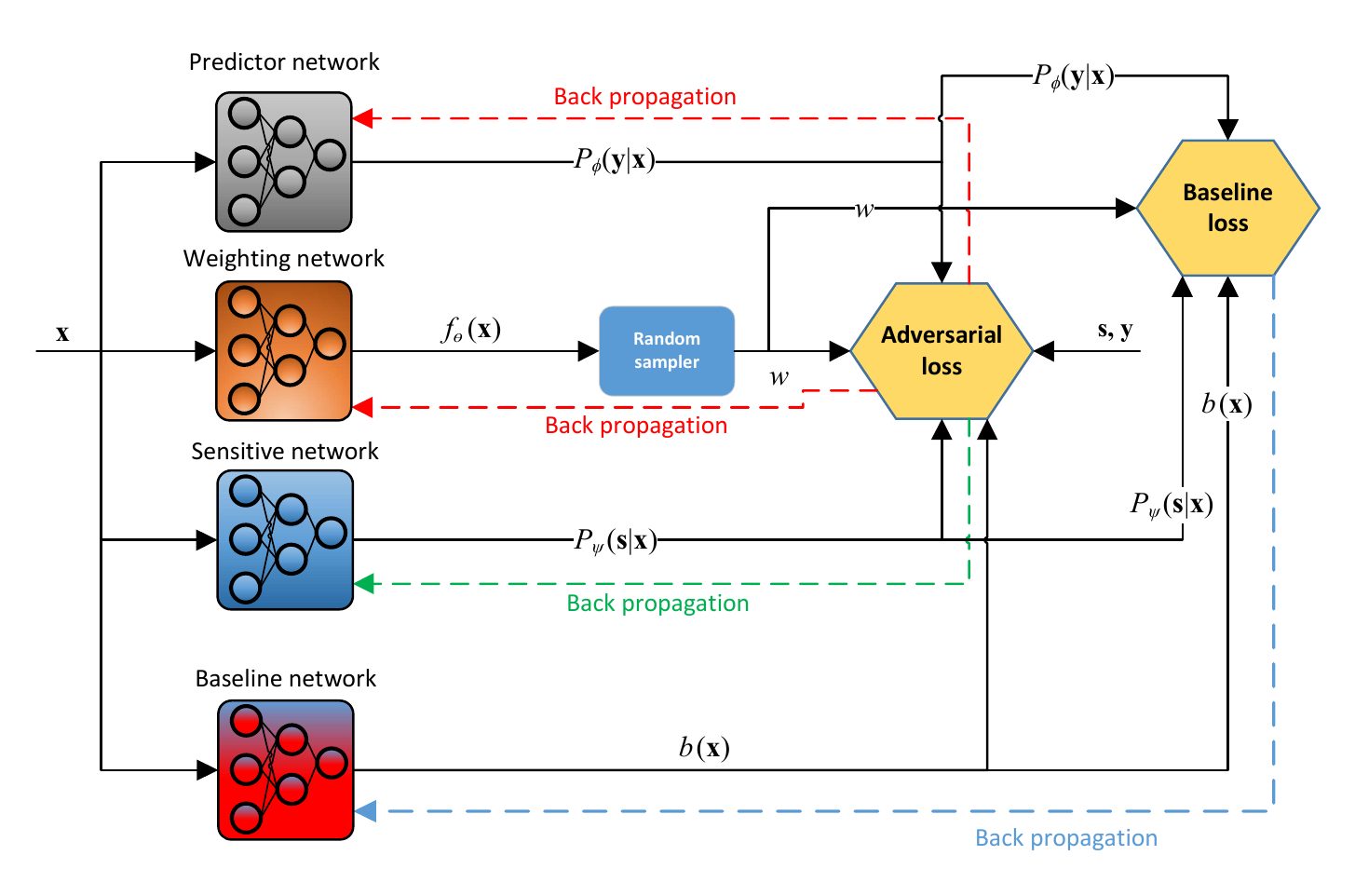}
	\captionsetup{justification=centering}
	\caption{Graphical representations of FAIR-betaSF model with baseline function}
	\label{fig:Fig2}
\end{figure*}

\begin{algorithm}[h!]
	\caption{FAIR-betaSF with baseline}
	\label{alg:beta-base}
	\begin{algorithmic}
		\State {\bfseries Input:} learning rates $\gamma_\theta, \gamma_\phi, \gamma_\psi$, $\gamma_b$ dataset $D$,  hyperparameter $\alpha$, number of iterations $M$
		\State {\bfseries Output:} parameters $\theta,\phi,\psi$, $\mu$
		\vspace{2mm}
		\State Initialize $\theta$, $\phi$, $\psi$, $\mu$
		\For{i = 1 to M}
		\State Sample a mini-batch $B\subseteq D$
		\State $\alpha_\mathbf{x},\beta_\mathbf{x}\leftarrow f_\theta(\mathbf{x})$ for each $\mathbf{x}\in B$
		\State Sample $w_\mathbf{x}\sim \beta(\alpha,\beta)$ for each $\mathbf{x}\in B$
		\State $d_{\theta} \leftarrow \gamma_\theta\frac{1}{|B|} \sum_{(\mathbf{x},\mathbf{y},\mathbf{s})\in B}\left[w_\mathbf{x}\nabla_{\theta}\log P_\theta(w_\mathbf{x}|\mathbf{x})\cdot\right.$
		
		\hspace{1.5cm}$\left.(\alpha\log P_\psi(\mathbf{s}|\mathbf{x})-\log P_\phi(\mathbf{y}|\mathbf{x})-b_\mu(\textbf{x}))\right]$
		\State $d_{\phi} \leftarrow \gamma_\phi\nabla_{\phi}{\cal L}^{\cal P}_\alpha(\theta,\phi,\psi,B)$
		\State $d_{\psi} \leftarrow -\gamma_\psi\nabla_{\psi}{\cal L}^{\cal P}_\alpha(\theta,\phi,\psi,B)$
		\State $d_{\mu} \leftarrow -\gamma_\mu\frac{1}{|B|} \sum_{(\mathbf{x},\mathbf{y},\mathbf{s})\in B}\left[w_\mathbf{x}\nabla_\mu b_\mu(\mathbf{x})\cdot\right.$
		
		\hspace{1.5cm}$\left.(\alpha\log P_\psi(\mathbf{s}|\mathbf{x})-\log P_\phi(\mathbf{y}|\mathbf{x}) - b_\mu(\mathbf{x}))\right]$
		\State $(\theta,\phi,\psi,\mu) \leftarrow (\theta,\phi,\psi,\mu) - (d_{\theta},d_{\phi}, d_{\psi},  d_{\mu})$
		\EndFor
	\end{algorithmic}
\end{algorithm}

\iffalse
In addition, to reduce the variance of gradient of loss in FAIR-Bernoulli and FAIR-betaSF, it is possible to incorporate additional baseline network $b_\mu(\mathbf{x})$. Please note that we did not evaluate this improvement, but we explain how it could be done. The goal of baseline network is to minimize the loss ${\cal L}_{\alpha}(\mu)$ given by:
\begin{equation}
\label{Eq.Baseline-loss}
%\begin{split}
\mathbb{E}_{\substack{\mathbf{x}, \mathbf{s}, \mathbf{y}\sim P(\mathbf{x},\mathbf{s}, \mathbf{y})\\w \sim P_\theta(w|\mathbf{x})}} \Big[ \vphantom{)^2}\big(
w \cdot ( \alpha \log P_\psi(\mathbf{s}|\mathbf{x}) - \log P_\phi(\mathbf{y}|\mathbf{x})) - b_\mu(\mathbf{x})\big)^2 \Big]
%\end{split}
\end{equation}

The adversarial loss ${\cal L}_{\alpha}(\theta, \phi, \psi)$ with baseline is given by:
\begin{equation}
\label{Eq.ADVBaseline-loss}
%\begin{split}
\mathbb{E}_{\substack{\mathbf{x}, \mathbf{s}, \mathbf{y}\sim P(\mathbf{x},\mathbf{s}, \mathbf{y})\\w \sim P_\theta(w|\mathbf{x})}} \Big[ \big(
w \cdot ( \alpha \log P_\psi(\mathbf{s}|\mathbf{x}) - \log P_\phi(\mathbf{y}|\mathbf{x})) - b_\mu(\mathbf{x})\big)\Big]
%\end{split}
\end{equation}
The optimization procedure in models with baseline is performed by alternating minimization of these two losses.
Pseudocode of FAIR-betaSF with baseline and derivation of baseline loss function are presented in Appendix \ref{app:baselines}.
\fi

\subsection{Analysis of model properties}

In order to analyze properties of all our models in a uniform manner, we discuss instance weights as real values in the interval $[0,1]$ and we emphasize dependence of the weight on the instance as $w_\mathbf{x}$ without explicating specifics of the dependence. Vector of all such weights is denoted $\mathbf{w}$ and it is denoted $\mathbf{w}^*$ if it is a part of the optimal solution of the corresponding adversarial problem.
In practice, expectations are approximated by sample means (or sums since outmost constant factors are irrelevant in optimization), and losses are regularized. Therefore we consider a regularized loss ${\cal L}_{\alpha}(\mathbf{w}, \phi, \psi)$:
\begin{equation}
\label{Eq:Loss-scalar-new}
\begin{split}
\sum_{(\mathbf{x},\mathbf{y},\mathbf{s})\in D} w_\mathbf{x} [\alpha\log P_{\psi}(\mathbf{s}|\mathbf{x}) -  \log P_{\phi}(\mathbf{y}|\mathbf{x})]\\
\text{s.t. }\|\theta\|^2_2+\|\phi\|^2_2+\|\psi\|^2_2\leq \lambda
\end{split}
\end{equation}
where dependence of $w_\mathbf{x}$ on $\theta$ is not made explicit, but we are aware that it exists.
To shorten the proofs, we formulate regularization in a constraint based manner \cite{tibshirani1996regression}, although it is more often formulated and implemented in a mathematically equivalent penalty based manner (note that the meaning of regularization parameter is reversed -- in penalty based formulation case $\lambda=0$ corresponds to an infinite value of $\lambda$ in constraint based formulation).

Now we analyze how our method behaves with respect to the hyperparameter $\alpha$. First, we aim to understand how it controls the trade-off between fairness and the quality of prediction of the target variable. This aspect is important for the practical application of the method. In a nutshell, extreme case $\alpha=0$ represents extreme emphasis on fairness and $\alpha\rightarrow\infty$ represents extreme emphasis on quality of prediction and disregard for fairness. Please note that a superficial glance at the adversarial problem would suggest vice versa, but we stress that it is not the case. The role of hyperparameter $\alpha$ in FAIR model is the opposite to its role in FAD model. Second, we aim to understand how the hyperparameter $\alpha$ affects the optimal weights assigned to the instances. It turns out that under some (reasonable) conditions the optimal weights will tend to $0$ and $1$ and that the value of $\alpha$ controls the proportion of the two limiting values. Further discussion is provided after the theoretical results.

\begin{lemma}
	If $\lambda$ is finite, there exist strictly negative constants
	$c_\phi$, $c'_\phi$, $c_\psi$, and $c'_\psi$ such that it holds
	$c_\phi\leq \log P_\phi(\mathbf{y}|\mathbf{x})\leq c'_\phi$ and
	$c_\psi\leq \log P_\phi(\mathbf{s}|\mathbf{x})\leq c'_\psi$ for
	any $\mathbf{x}$, $\mathbf{y}$, and $\mathbf{s}$, and any
	$\phi$ and $\psi$ which satisfy regularization condition \ref{Eq:Loss-scalar-new}.
	\label{pp:bounded}
\end{lemma}
\begin{proof}
	Denote ${\cal B}$ the ball defined by $\|\theta\|^2_2+\|\phi\|^2_2+\|\psi\|^2_2\leq \lambda$, representing the set of feasible solutions of the optimization problem.
	Denote $\bar{g}_\phi(\mathbf{x})$ the network $g_\phi(\mathbf{x})$ modelling $\mathbf{y}$ with sigmoid function at the output removed and $\bar{h}_\psi(\mathbf{x})$ the network $h_\psi(\mathbf{x})$ modelling $\mathbf{s}$ with sigmoid at the output removed. Since ${\cal B}$ is a compact set and $\bar{g}_\phi(\mathbf{x})$ and $\bar{h}_\psi(\mathbf{x})$ are continuous functions, they both attain their finite minimal and maximal values within ${\cal B}$. Since $\log P_\psi(\mathbf{s}|\mathbf{x})$ and $\log P_\phi(\mathbf{y}|\mathbf{x})$ are continuous functions of  $\bar{h}_\psi(\mathbf{x})$ and $\bar{g}_\phi(\mathbf{x})$, respectively, which map the range of $\bar{h}_\psi$ and $\bar{g}_\phi$ from $(-\infty,\infty)$ to $(-\infty, 0)$, functions $\log P_\psi(\mathbf{s}|\mathbf{x})$ and $\log P_\phi(\mathbf{y}|\mathbf{x})$ attain their strictly negative and finite minimal and maximal values within ${\cal B}$. Therefore, the required constants exist, by which the lemma is proven.
\end{proof}

\begin{theorem}
	If $\lambda$ is finite, for $\alpha=0$ it holds $\mathbf{w}^*=\mathbf{0}$.
	\label{pp:zero}
\end{theorem}
\begin{proof}
	By Lemma \ref{pp:bounded}, $P_{\psi}(\mathbf{s}|\mathbf{x})$ is bounded, so for $\alpha=0$ it holds:
	\begin{align*}
	(\mathbf{w}^*,\phi^*,\psi^*)&=\arg\min_{\mathbf{w},\phi}\max_{\psi}{\cal L}_\alpha(\mathbf{w},\phi,\psi)\\
	&=\arg\min_{\mathbf{w},\phi}- \sum_{(\mathbf{x},\mathbf{y},\mathbf{s})\in D} w_\mathbf{x} \cdot \log P_{\phi}(\mathbf{y}|\mathbf{x})
	\end{align*}
	By Lema \ref{pp:bounded}, $\log P_{\phi}(\mathbf{y}|\mathbf{x})$ is strictly negative, so $- \sum_{(\mathbf{x},\mathbf{y},\mathbf{s})\in D} w_\mathbf{x} \cdot \log P_{\phi}(\mathbf{y}|\mathbf{x})$ is zero or positive. Therefore its minimal value is $0$ for $\mathbf{w}=\mathbf{0}$ regardless of $\phi$.
	Therefore, it holds $\mathbf{w}^*=\mathbf{0}$.
\end{proof}

\begin{theorem}
	For each instance $(\mathbf{x},\mathbf{y},\mathbf{s})$, it holds
	$w_{\mathbf{x}}^*= 1$ or $w_{\mathbf{x}}^*= 0$ or $\alpha\log P_{\psi^*}(\mathbf{s}|\mathbf{x})=\log P_{\phi^*}(\mathbf{y}|\mathbf{x})$.
	\label{pp:anypoint}
\end{theorem}
\iffalse
\begin{proof}
	Consider a partial derivative of ${\cal L}_\alpha$ for fixed but arbitrary $\phi$ and $\psi$:
	$$\frac{\partial{\cal L}_\alpha}{\partial w_\mathbf{x}}(\mathbf{w},\phi,\psi)=\alpha\log P_{\psi}(\mathbf{s}|\mathbf{x})-\log P_{\phi}(\mathbf{y}|\mathbf{x})$$
	This derivative can be equal to $0$ (in which case it holds $\alpha\log P_{\psi}(\mathbf{s}|\mathbf{x})=\log P_{\phi}(\mathbf{y}|\mathbf{x})$), negative, or positive.
	If the derivative is negative, for those fixed $\phi$ and $\psi$, ${\cal L}_\alpha(\mathbf{w},\phi,\psi)$ is a strictly decreasing function of $w_\mathbf{x}$, so its minimum is attained for $w_\mathbf{x}=1$. If this derivative is positive, minimal value is attained for $w_\mathbf{x}=0$. Therefore, for arbitrary fixed $\phi$ and $\psi$, it either holds $\alpha\log P_{\psi}(\mathbf{s}|\mathbf{x})=\log P_{\phi}(\mathbf{y}|\mathbf{x})$ or it holds that the minimum is attained for $w_\mathbf{x}=1$ or for $w_\mathbf{x}=0$. Since this holds for arbitrary $\phi$ and $\psi$, it also holds for optimal choices $\phi^*$, $\psi^*$, by which the theorem is proven.
\end{proof}
\fi

\begin{proof}
	Consider a partial derivative in the optimal solution:
	$$\frac{\partial{\cal L}_\alpha}{\partial w_\mathbf{x}}(\mathbf{w}^*,\phi^*,\psi^*)=\alpha\log P_{\psi^*}(\mathbf{s}|\mathbf{x})-\log P_{\phi^*}(\mathbf{y}|\mathbf{x})$$
	If the derivative is negative, then there exists $d>0$ such that it holds
	$${\cal L}_\alpha(\mathbf{w}^*+d\mathbf{e}_{\mathbf{x}},\phi^*,\psi^*) < {\cal L}_\alpha(\mathbf{w}^*,\phi^*,\psi^*)$$
	where $\mathbf{e}_\mathbf{x}=(0,\ldots,1,\ldots,0)\in\mathbb{R}^{|D|}$ where $1$ is at the coordinate corresponding to $w_\mathbf{x}$.
	Therefore, if it holds $w_\mathbf{x}^*< 1$, $w_\mathbf{x}^*$ can be increased in order to decrease the loss and $(\mathbf{w}^*,\phi^*,\psi^*)$ is not an optimal solution, which is a contradiction. Therefore, it has to hold $w_\mathbf{x}^*=1$.
	If the derivative is positive, $w_\mathbf{x}^*=0$ is proven in an analogous manner.
	If the derivative is $0$, the theorem holds due to its third case.
\end{proof}

In the following propositions, we explicitly denote dependence of the optimal solution on $\alpha$.

\begin{lemma}
	If $\lambda$ is finite, for each instance $(\mathbf{x},\mathbf{y},\mathbf{s})\in D$ it holds
	$$\frac{\partial{\cal L}_\alpha}{\partial w_\mathbf{x}}(\mathbf{w}^*_\alpha,\phi^*_\alpha,\psi^*_\alpha)\rightarrow -\infty\hspace{5mm}\text{as}\hspace{5mm}\alpha\rightarrow\infty$$
	\label{pp:regularization}
\end{lemma}
\begin{proof}
	Consider a partial derivative with respect to $w_\mathbf{x}$ in an optimum:
	$$\frac{\partial{\cal L}_\alpha}{\partial w_\mathbf{x}}(\mathbf{w}^*_\alpha,\phi^*_\alpha,\psi^*_\alpha)=\alpha\log P_{\psi^*_\alpha}(\mathbf{s}|\mathbf{x})-\log P_{\phi^*_\alpha}(\mathbf{y}|\mathbf{x})$$
	According to Lemma \ref{pp:bounded}, for any feasible $\psi$ and $\phi$ there exists constants $c'_\psi<0$ and $c_\phi$ such that it holds $\log P_{\psi}(\mathbf{s}|\mathbf{x})\leq c'_\psi$ and $\log P_{\phi}(\mathbf{y}|\mathbf{x})\geq c_\phi$. Therefore, the first term goes to $-\infty$ as $\alpha\rightarrow\infty$ and the second term is bounded, so the limit of the partial derivative is $-\infty$.
\end{proof}

\begin{theorem}
	If $\lambda$ is finite, for each instance $(\mathbf{x},\mathbf{y},\mathbf{s})\in D$, it holds $w_{\mathbf{x},\alpha}^*\rightarrow 1$ as $\alpha\rightarrow\infty$.
	\label{pp:infty}
\end{theorem}
\begin{proof}
	According to Lemma \ref{pp:regularization},
	the limit of the values of the partial derivative $\frac{\partial{\cal L}_\alpha}{\partial w_\mathbf{x}}(\mathbf{w}^*_\alpha,\phi^*_\alpha,\psi^*_\alpha)$ in optima as $\alpha\rightarrow\infty$ is negative. Then, by the definition of the limit, there exists $\alpha_0\in \mathbb{R}$ such that for all $\alpha>\alpha_0$ it holds
	$$\frac{\partial{\cal L}_\alpha}{\partial w_\mathbf{x}}(\mathbf{w}^*_\alpha,\phi^*_\alpha,\psi^*_\alpha)<0$$
	For each such $\alpha$, since derivative with respect to $w_\mathbf{x}$ is negative, by the same argument as in the proof of Theorem \ref{pp:anypoint}, it holds $w^*_{\mathbf{x},\alpha}=1$. Hence, we can conclude that for each $\varepsilon>0$, there exists $\alpha_0$ such that for all $\alpha>\alpha_0$ it holds $w^*_{\mathbf{x},\alpha}>1-\varepsilon$  (since $w^*_{\mathbf{x},\alpha}=1$). Therefore, by the definition of the limit, we conclude that it holds $w^*_{\mathbf{x},\alpha}\rightarrow 1$ as $\alpha\rightarrow\infty$.
\end{proof}

In case of infinite $\lambda$, overfitting might falsify our proof of Lemma \ref{pp:bounded} and in that case for some instance $\mathbf{x}$ it might hold $w^*_{\mathbf{x},\alpha}\rightarrow 0$ as $\alpha\rightarrow\infty$. However, this suggests an interesting diagnostic property -- if for ever larger values of $\alpha$ one obtains $w_\mathbf{x}=0$ for some $\mathbf{x}$, one has reasons to suspect overfitting. Also, finite capacity of the network might make regularization 
unnecessary in practice. However, theoretical analysis was easier under the assumption of explicit regularization.

Also note that the model of instance weights does not need allow values $0$ and $1$. Nevertheless, the provided theorems inform us that the gradients will push the weights towards these values. Still, our probabilistic approaches might provide additional regularization by giving nonzero probability to other weight values except the optimal ones.

Provided theorems explain the role of hyperparameter $\alpha$ in our model -- it is a threshold on the ratio of instance usefulness and instance fairness based on which the model decides if the instance should be discarded or used for learning.
If it holds
$$\frac{\log P_{\phi}(\mathbf{y}|\mathbf{x})}{\log P_{\psi}(\mathbf{s}|\mathbf{x})}<\alpha$$
intuitively, the instance is fair enough considering its usefulness. Namely, for the ratio to be low, its predictive usefulness should be high (reflected by small negative value of log likelihood in the numerator) and its unfairness should be low (reflected by the large negative value of log likelihood in the denominator). In the extreme case of $\alpha=0$ no instance is considered fair enough, since neither the log likelihood in the numerator can be exactly zero, nor the log likelihood in the denominator can be infinite. According to Theorem \ref{pp:zero}, in that case, all instances are discarded. In the other extreme, according to Theorem \ref{pp:infty}, as $\alpha$ tends to infinity,  fairness is disregarded and all instances are used for learning. For values of $\alpha$ in between some instances are disregarded and some are used.

\section{Experimental Setup}
\label{Sec:exp-evaluation}

\textbf{Datasets}. The proposed framework was tested on four datasets, three of which are commonly used benchmarks. Two datasets (German credit and Adult income) come from the UCI ML repository \cite{frank2011uci}. To our knowledge the Hospital readmission dataset was used in this paper for the first time in the context of fairness.

The first, the \textit{Adult income} dataset \cite{kohavi1996scaling} represents a binary classification task of predicting whether an income is greater than 50K dollars. The dataset contains 45,222 instances described by 14 features and including the sensitive attribute Gender. The atributes used in the dataset describes the individual's education level, age, gender, occupation, workclass, martial-status, relationship, capital loss and etc \cite{Dua:2019}. After applying dummy coding, total number of features was 93. Total numbers of instances used in training, validation and testing are 31,655, 6,783, and 6,784, respectively.

Second dataset we used is the \textit{Hospital readmission} dataset \cite{stiglic2015comprehensible}. It represents a binary classification task where label 1 means that patient is readmitted within 30 days. The dataset consists of 66,994 instances and 931 attributes, including sensitive attribute Gender. Total number of instances used in training, validation and testing are 46,895, 10,049 and 10,050, respectively.

The third dataset, named \textit{Hospital Expenditures}, comes from \cite{bellamy2019ai}. It represents a binary classification task of predicting whether a person would have high or low utilization of medical expenditures. The sensitive attribute is Race. Dataset contains 15,830 instances and 133 attributes, after dummy codding, total number of attributes used in this dataset was 138. For training, validation, and testing, we used 11,081, 2,374 and 2,375 instances respectively.

As a fourth dataset, we used \textit{German credit} dataset. \textit{German credit} dataset has 1,000 instances where the task is to classify bank account holders into classes good or bad. The total number of attributes used in the dataset, after applying dummy coding is 58, including sensitive attributes. Following the definition of fairness from \cite{kamiran2012decision} for German credit dataset, there are two sensitive attributes, one being Gender and other being Age ($\geq 25$ is considered as privileged class, and $< 25$ as unprivileged class).  Total numbers of instances used in training, validation and testing was 700, 150, and 150, respectively.

\textbf{Models}. \sloppy The results obtained by FAIR models are compared with seven related and state-of-the-art algorithms: FAD, its probabilistic variant (FAD-prob), reweighing preprocessing technique from~\cite{kamiran2012data} combined with the random forest classifier (Reweighing - RF) and with neural networks (Reweighing - NN), disparity impact remover~\cite{feldman2015certifying} combined with random forest (DI - RF) and neural networks (DI - NN) and prejudice remover~\cite{kamishima2012fairness} (PR). Architecture, number of epochs in early stopping procedure, and learning rates were empirically determined as to optimize the performance of each model, by varying design choices of the architectures described in the literature. Detailed specifications can be found in \ref{app:Architecture}. We did not use explicit regularization in our experiments since, in accordance with the remark after the proof of theorem \ref{pp:infty}, capacity of the models can also be controlled through the choice of architecture.

\textbf{Optimization}.
For optimization of all neural network based models we use Adam optimizer \cite{bock2019proof}.
During optimization of FAIR and FAD models, early stopping was used. In the early stopping procedure, the min and max objectives of adversarial training procedure on validation set were monitored. In case when there were no improvements in either of these two metrics for a given number of epoch (provided in \ref{app:Architecture}), the training procedure is stopped.

\textbf{Metrics}.
Classification performance of all presented classifiers is quantified by the area under the ROC curve (AUC), which is calculated for the target variable ($\mathbf{y}$) and the sensitive attribute ($\mathbf{s}$). Therefore, we present $AUC_y$ and $AUC_s$ for the target variable and the sensitive attribute, respectively. If subscript is omitted, then $AUC_y$ is presented. As fairness metrics we use ASD, AEOD, and AOD defined by Eqs. \ref{metric:asd}, \ref{metric:aeod}, and \ref{metric:aod}, respectively. 

\textbf{Evaluation procedure and presentation of results}. The evaluated models (both FAIR and the baselines) have hyperparameters which affect the trade-off between fairness and predictive performance of the classifiers. Note that such hyperparameters do not control model capacity. Therefore, we do not tune them to obtain maximal perfomance (like one might tune regularization hyperparameters). Instead, we vary them in order to illustrate model behaviour for different trade-offs. The hyperparameters $\alpha$ of FAIR and FAD models were varied in range $[0, 10^{-3}, 10^{-2}, 10^{-1}, 1, 10, 10^{2}, 10^{3}]$, whereas in the case of other models, hyperparameters were varied in range $[0, 10^{-3}, 10^{-2}, 10^{-1}, 1]$ (since $1$ is the maximal value for these methods). For each such value, evlauation metrics were computed. We call a set of models obtained from one model kind (FAD, FAIR, etc.) by varying the fairnes related hyperparamter, a {\em model family}. For instance, all FAD models trained for different values of $\alpha$ constitute a FAD model family.

Since the models are evaluated by two criteria (predictive performance of the target variable and fairness), one model can be better than the other according to one criterion and vice-versa. Since both critera are important, instead of privileging one of them we present our results in terms of Pareto fronts. For a set of trained models, Pareto front consists of models which are not dominated by any other models in terms of both predictive performance and fairness \cite{marler2004survey}. Models which are dominated by others in terms of both criteria are obviously irrelevant and should be discarded. Pareto front can be plotted in 2D in terms of metrics for the two criteria used and visually inspected. We are interested in two
kinds of Pareto fronts:
\begin{itemize}
\item The overall Pareto front which is a Pareto front of all trained models (union of all model familites). Models which yield more points in such Pareto front are better.
\item Per family Pareto fronts, which are Pareto fronts obtained separately for different model families. They provide a more detailed view of models.
\end{itemize}

Construction of the Pareto front includes model selection - models are compared according to their performance and some of them are selected. Since evaluation metrics should never be reported on the data on which the selection was performed, we take care to train all models on the training set, to perform selection of the models for the Pareto front on the validation set, and to evaluate selected models on the test set. All results reported in the following section are calculated on the test set.

\section{Results and Discussion}
\label{sec:ResultsDiscussion}

In this section we provide experimental results following the above provided setup. Further on, we provide the discussion of these results and the qualitative evaluation of the behaviour of our model.

\subsection{Results}
\label{sec:Results}

In this section we present results obtained using the experimental evaluation outlined above.

\begin{figure*}
	\center
	\includegraphics[angle=0, width=0.7\textwidth]{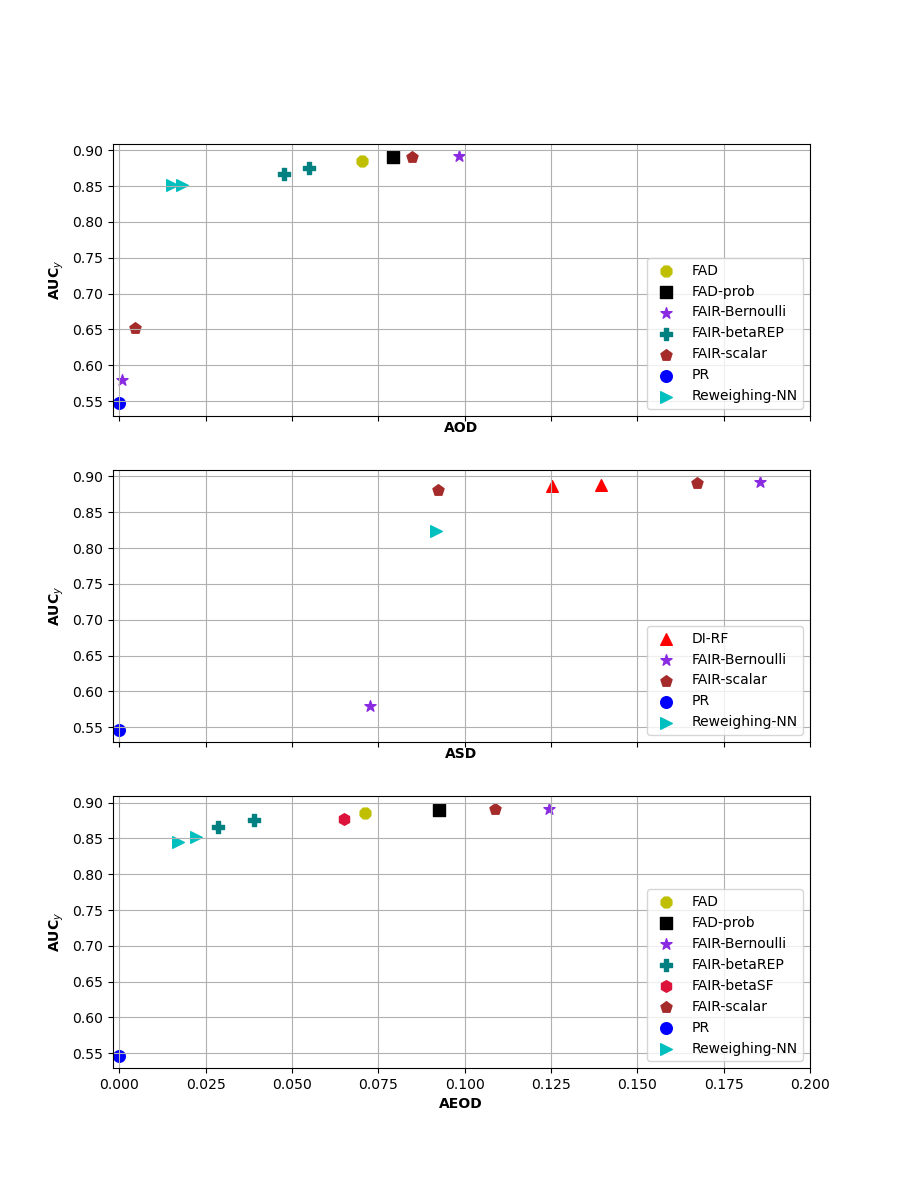}
	\captionsetup{justification=centering}
	\caption{Classification performance and fairness of models as measured by $\mathbf{AUC_y}$ and $\mathbf{AOD}$, $\mathbf{ASD}$ or $\mathbf{AEOD}$ on the \textit{Adult income} datasets}
	\label{fig:Adult}
	\vskip -0.2in
\end{figure*}

Firstly, models performances obtained on \textit{Adult income} dataset are illustrated in Fig.~\ref{fig:Adult} by three fairness metrics ($\mathbf{AOD}$, $\mathbf{ASD}$ or $\mathbf{AEOD}$) and classification performance ($\mathbf{AUC_y}$). The models with greater AUC score and lower (un)fairness metric (upper left corner of plots) are preferred. It can be observed that FAIR models dominates Pareto optimal solutions with respect to the all fairness metrics. In addition, Reweighing-NN and FAIR-betaREP models dominate the upper left corner of Pareto fronts for AOD and AEOD, whereas the FAIR-scalar dominates the upper left corner of Pareto front for ASD metric. Moreover, in table~\ref{tab:Adult} the Pareto optimal solutions obtained for all three fairness metrics and ($\mathbf{AUC_y}$) are presented. Similarly, it can be concluded that the number of FAIR models is larger compared to the other models and FAIR can therefore be considered better then other models.

\begin{table}
	\centering
	\caption{Pareto optimal solutions - \textit{Adult income} dataset}
	\begin{tabular}{|c|l|l|l|l|}
		\hline
		\textbf{model}          & \multicolumn{1}{c|}{\textbf{AUC}} & \multicolumn{1}{c|}{\textbf{AOD}} & \multicolumn{1}{c|}{\textbf{ASD}} & \multicolumn{1}{c|}{\textbf{AEOD}} \\ \hline
		\textbf{PR}             & 0.547                             & 0.000                             & 0.000                             & 0.000                              \\ \hline
		\textbf{DI-RF}          & 0.887                             & 0.095                             & 0.125                             & 0.155                              \\ \hline
		\textbf{DI-RF}          & 0.888                             & 0.103                             & 0.139                             & 0.163                              \\ \hline
		\textbf{Reweighing-NN}  & 0.852                             & 0.018                             & 0.152                             & 0.022                              \\ \hline
		\textbf{Reweighing-NN}  & 0.849                             & 0.020                             & 0.123                             & 0.005                              \\ \hline
		\textbf{Reweighing-NN}  & 0.824                             & 0.040                             & 0.092                             & 0.117                              \\ \hline
		\textbf{Reweighing-NN}  & 0.851                             & 0.015                             & 0.111                             & 0.064                              \\ \hline
		\textbf{Reweighing-RF}  & 0.888                             & 0.104                             & 0.167                             & 0.143                              \\ \hline
		\textbf{FAD}            & 0.886                             & 0.070                             & 0.172                             & 0.071                              \\ \hline
		\textbf{FAIR-scalar}    & 0.652                             & 0.005                             & 0.094                             & 0.052                              \\ \hline
		\textbf{FAIR-scalar}    & 0.881                             & 0.115                             & 0.092                             & 0.154                              \\ \hline
		\textbf{FAIR-scalar}    & 0.891                             & 0.085                             & 0.167                             & 0.109                              \\ \hline
		\textbf{FAIR-betaSF}    & 0.867                             & 0.061                             & 0.146                             & 0.063                              \\ \hline
		\textbf{FAIR-betaSF}    & 0.878                             & 0.074                             & 0.197                             & 0.065                              \\ \hline
		\textbf{FAIR-betaSF}    & 0.867                             & 0.055                             & 0.142                             & 0.068                              \\ \hline
		\textbf{FAIR-Bernoulli} & 0.580                             & 0.001                             & 0.073                             & 0.075                              \\ \hline
		\textbf{FAIR-Bernoulli} & 0.892                             & 0.098                             & 0.185                             & 0.124                              \\ \hline
		\textbf{FAIR-betaREP}   & 0.875                             & 0.055                             & 0.174                             & 0.039                              \\ \hline
		\textbf{FAIR-betaREP}   & 0.866                             & 0.048                             & 0.159                             & 0.029                              \\ \hline
		\textbf{FAD-prob}       & 0.890                             & 0.079                             & 0.172                             & 0.093                              \\ \hline
	\end{tabular}
	\label{tab:Adult}
\end{table}

Secondly, the results obtained on \textit{Hospital readmission} dataset are presented in Fig~\ref{fig:Readmission}. It can be noticed that only FAIR models exist on Pareto front and consequently all other models are dominated by them. It can be observed that in the case of AOD and ASD metrics FAIR-scalar is the closest to the upper left corner and can therefore be considered better than others. The letter can be also confirmed in the table~\ref{tab:Readmission} where only FAIR models exist on overall Pareto front.

\begin{figure}
	\center
	\includegraphics[angle=0, width=0.7\textwidth]{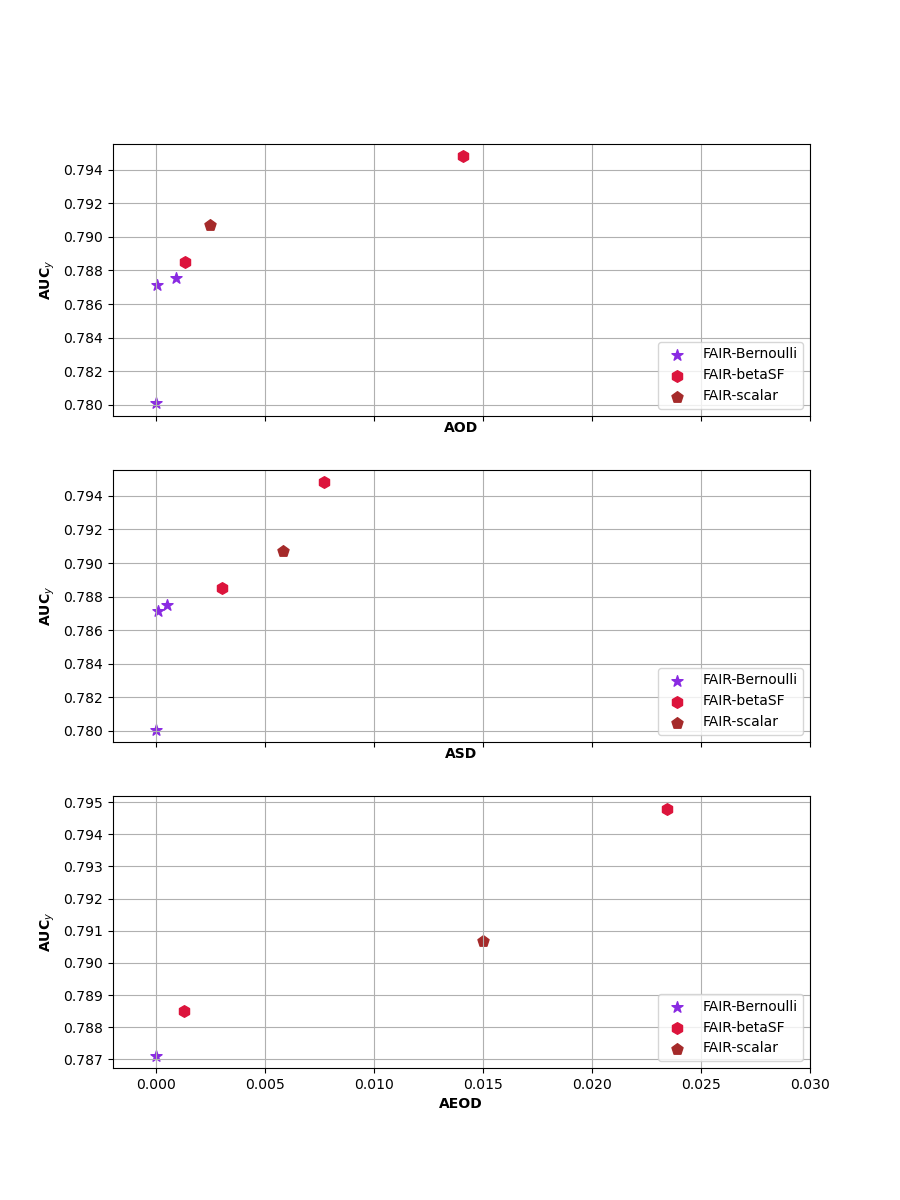}
	\captionsetup{justification=centering}
	\caption{Classification performance and fairness of models as measured by $\mathbf{AUC_y}$ and $\mathbf{AOD}$, $\mathbf{ASD}$ or $\mathbf{AEOD}$ on the \textit{Hospital readmission} dataset}
	\label{fig:Readmission}
	\vskip -0.2in
\end{figure}

\begin{table}
	\centering
	\caption{Pareto optimal solutions - \textit{Hospital readmission} dataset}
	\begin{tabular}{|c|l|l|l|l|}
		\hline
		\textbf{model}          & \multicolumn{1}{c|}{\textbf{AUC}} & \multicolumn{1}{c|}{\textbf{AOD}} & \multicolumn{1}{c|}{\textbf{ASD}} & \multicolumn{1}{c|}{\textbf{AEOD}} \\ \hline
		\textbf{FAIR-scalar}    & 0.791                             & 0.002                             & 0.006                             & 0.015                              \\ \hline
		\textbf{FAIR-betaSF}    & 0.795                             & 0.014                             & 0.008                             & 0.023                              \\ \hline
		\textbf{FAIR-betaSF}    & 0.789                             & 0.001                             & 0.003                             & 0.001                              \\ \hline
		\textbf{FAIR-Bernoulli} & 0.780                             & 0.000                             & 0.000                             & 0.000                              \\ \hline
		\textbf{FAIR-Bernoulli} & 0.787                             & 0.000                             & 0.000                             & 0.000                              \\ \hline
		\textbf{FAIR-Bernoulli} & 0.788                             & 0.001                             & 0.000                             & 0.002                              \\ \hline
	\end{tabular}
	\label{tab:Readmission}%
\end{table}

Thirdly, models performances obtained on \textit{Hospital expenditures} dataset are illustrated in Fig.~\ref{fig:Adult}. It can be observed that the FAIR models dominates Pareto front in all presented metrics. In the case of ASD metric PR model is the closest to the upper left corner of Pareto front, whereas in the case of AOD and AEOD metrics similarly can be concluded for DI-RF model. However in table~\ref{tab:MEPS-19} where overall Pareto front is presented the FAIR models still dominates Pareto front.

\begin{figure}
	\center
	\includegraphics[angle=0, width=0.7\textwidth]{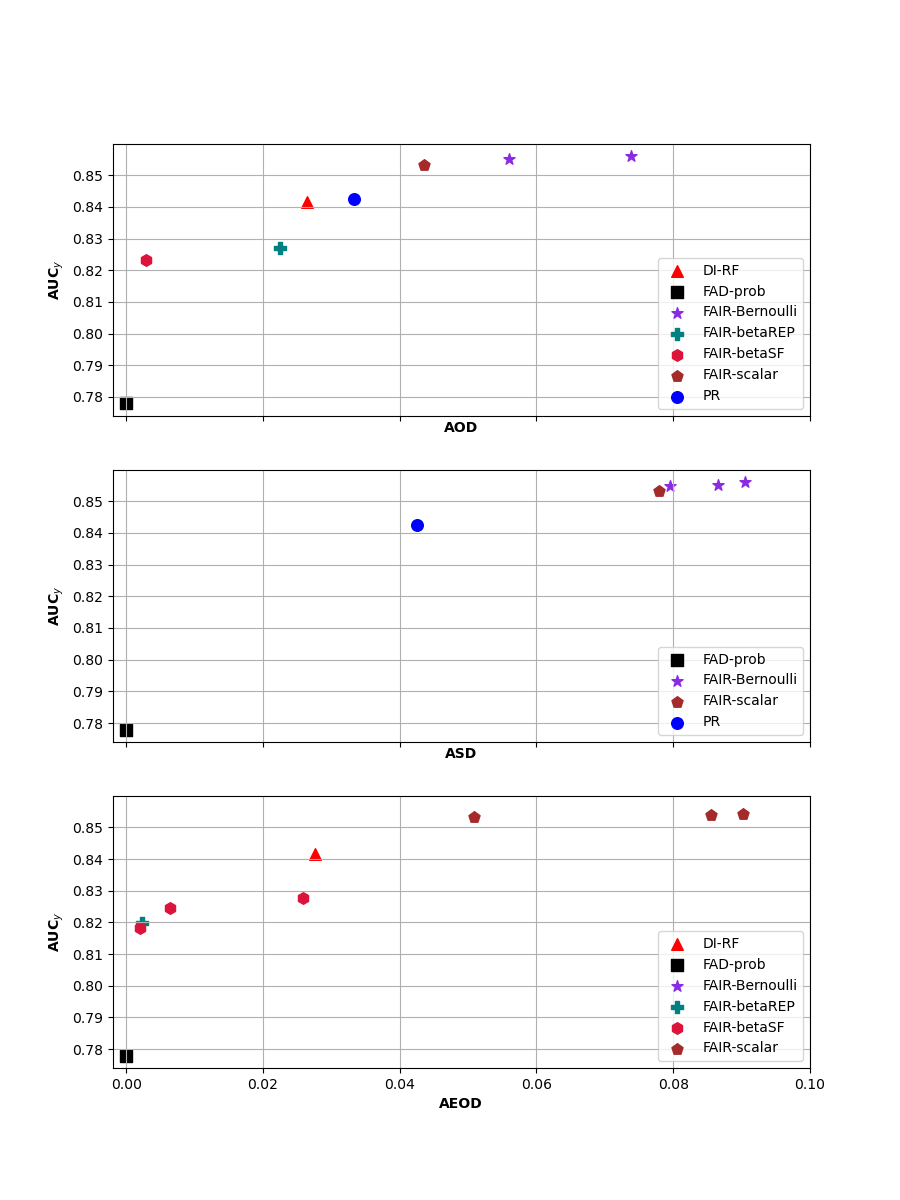}
	\captionsetup{justification=centering}
	\caption{Classification performance and fairness of models as measured by $\mathbf{AUC_y}$ and $\mathbf{AOD}$, $\mathbf{ASD}$ or $\mathbf{AEOD}$ \textit{Hospital expenditures} dataset}
	\label{fig:MEPS19}
	\vskip -0.2in
\end{figure}

\begin{table}
	\centering
	\caption{Pareto optimal solutions - \textit{Hospital expenditures} dataset}
	\begin{tabular}{|c|l|l|l|l|}
		\hline
		\textbf{model}          & \multicolumn{1}{c|}{\textbf{AUC}} & \multicolumn{1}{c|}{\textbf{AOD}} & \multicolumn{1}{c|}{\textbf{ASD}} & \multicolumn{1}{c|}{\textbf{AEOD}} \\ \hline
		\textbf{PR}             & 0.842                             & 0.033                             & 0.043                             & 0.059                              \\ \hline
		\textbf{DI-RF}          & 0.840                             & 0.032                             & 0.058                             & 0.043                              \\ \hline
		\textbf{DI-RF}          & 0.842                             & 0.027                             & 0.059                             & 0.028                              \\ \hline
		\textbf{Reweighing-RF}  & 0.842                             & 0.040                             & 0.065                             & 0.057                              \\ \hline
		\textbf{FAIR-scalar}    & 0.853                             & 0.044                             & 0.078                             & 0.051                              \\ \hline
		\textbf{FAIR-betaSF}    & 0.823                             & 0.003                             & 0.050                             & 0.013                              \\ \hline
		\textbf{FAIR-betaSF}    & 0.821                             & 0.013                             & 0.052                             & 0.007                              \\ \hline
		\textbf{FAIR-betaSF}    & 0.828                             & 0.028                             & 0.069                             & 0.026                              \\ \hline
		\textbf{FAIR-betaSF}    & 0.825                             & 0.015                             & 0.059                             & 0.006                              \\ \hline
		\textbf{FAIR-betaSF}    & 0.827                             & 0.024                             & 0.065                             & 0.021                              \\ \hline
		\textbf{FAIR-betaSF}    & 0.818                             & 0.012                             & 0.054                             & 0.002                              \\ \hline
		\textbf{FAIR-Bernoulli} & 0.855                             & 0.056                             & 0.087                             & 0.070                              \\ \hline
		\textbf{FAIR-Bernoulli} & 0.856                             & 0.074                             & 0.091                             & 0.110                              \\ \hline
		\textbf{FAIR-Bernoulli} & 0.855                             & 0.047                             & 0.080                             & 0.056                              \\ \hline
		\textbf{FAIR-betaREP}   & 0.820                             & 0.012                             & 0.053                             & 0.002                              \\ \hline
		\textbf{FAIR-betaREP}   & 0.820                             & 0.007                             & 0.052                             & 0.008                              \\ \hline
		\textbf{FAIR-betaREP}   & 0.827                             & 0.022                             & 0.055                             & 0.027                              \\ \hline
		\textbf{FAIR-betaREP}   & 0.817                             & 0.011                             & 0.056                             & 0.003                              \\ \hline
		\textbf{FAIR-betaREP}   & 0.816                             & 0.010                             & 0.057                             & 0.003                              \\ \hline
		\textbf{FAD-prob}       & 0.778                             & 0.000                             & 0.000                             & 0.000                              \\ \hline
	\end{tabular}
	\label{tab:MEPS-19}%
\end{table}

Eventually, model performances obtained of \textit{German credit} dataset for age and sex as sensitive attributes are presented in Figs.~\ref{fig:Ger-age-res} and~\ref{fig:Ger-sex-res}, respectively. Similarly as in previous datasets, overall Pareto fronts, for age and sex as sensitive attributes, are presented in Tables~\ref{tab:Ger-age} and~\ref{tab:Ger-sex}. It can be observed that in  Fig.~\ref{fig:Ger-sex-res} all models are equally represented, whereas in  Fig.~\ref{fig:Ger-age-res} FAIR models dominates the Pareto front in all cases. Similar conclusion can be made in the case of overall Pareto fronts that are presented in tables~\ref{tab:Ger-age} and~\ref{tab:Ger-sex}. In table~\ref{tab:Ger-age} it can be observed that all models are equally represented in Pareto fronts, whereas in table~\ref{tab:Ger-sex} the most dominant models are Reweighing-RF and FAD.

\begin{figure}
	\center
	\includegraphics[angle=0, width=0.7\textwidth]{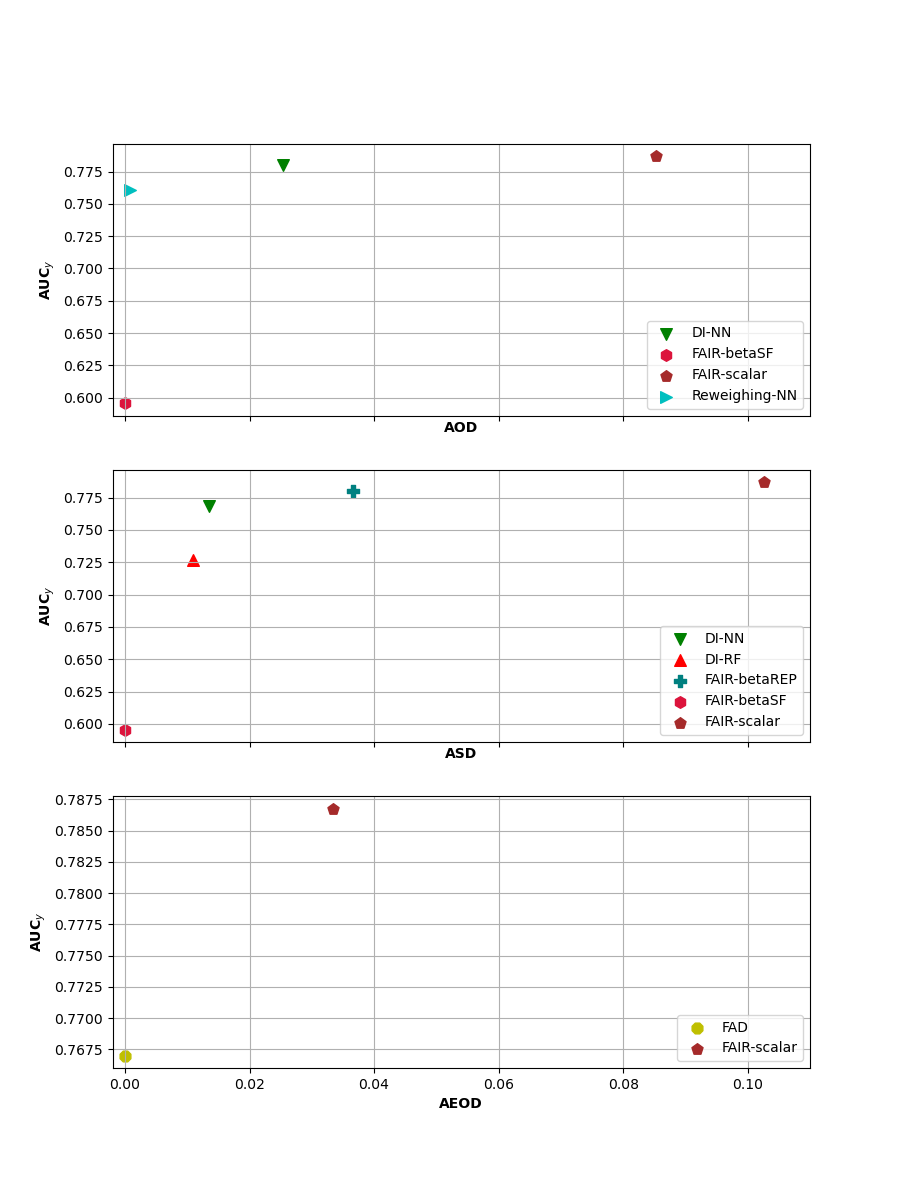}
	\captionsetup{justification=centering}
	\caption{Classification performance and fairness of models as measured by $\mathbf{AUC_y}$ and $\mathbf{AOD}$, $\mathbf{ASD}$ or $\mathbf{AEOD}$ \textit{German credit} (age) dataset}
	\label{fig:Ger-sex-res}
	\vskip -0.2in
\end{figure}

\begin{figure}
	\center
	\includegraphics[angle=0, width=0.7\textwidth]{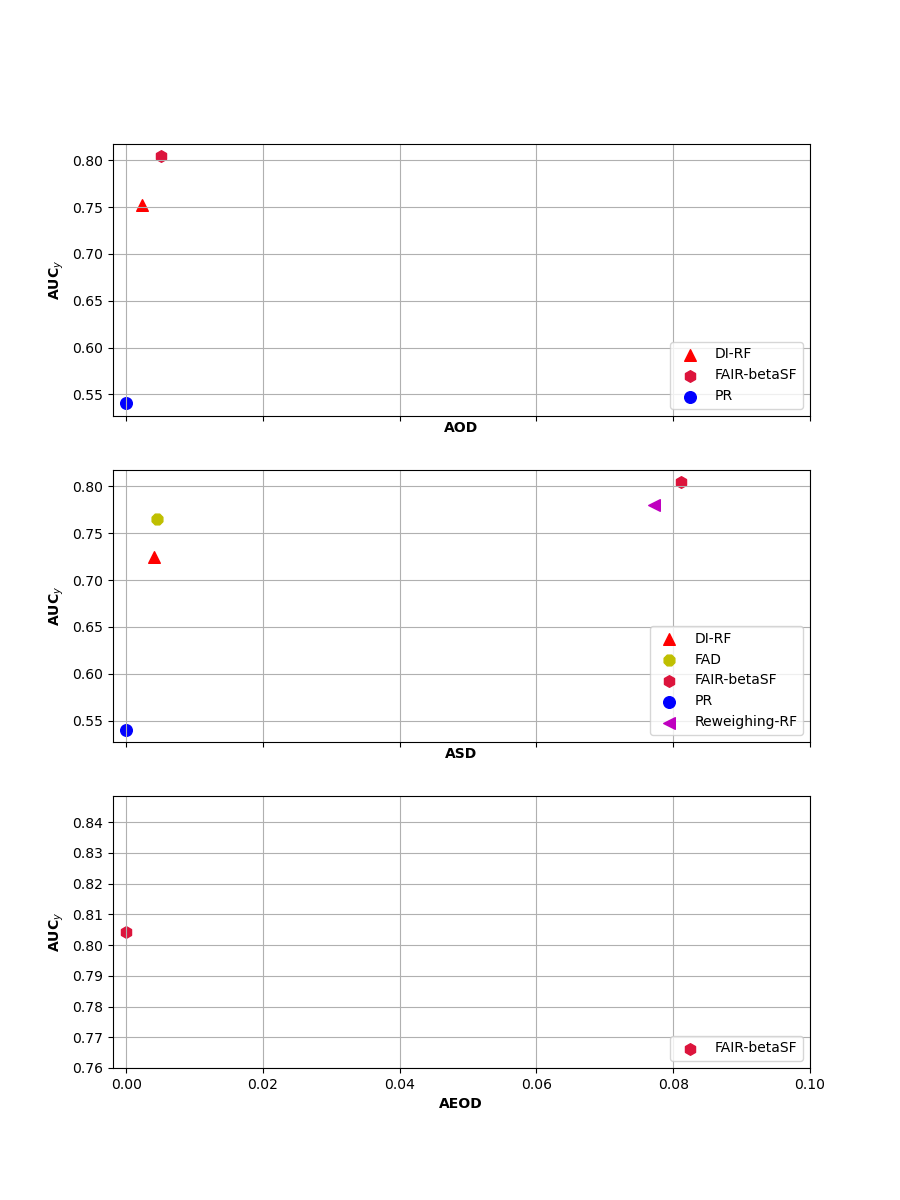}
	\captionsetup{justification=centering}
	\caption{Classification performance and fairness of models as measured by $\mathbf{AUC_y}$ and $\mathbf{AOD}$, $\mathbf{ASD}$ or $\mathbf{AEOD}$ \textit{German credit} (sex) dataset}
	\label{fig:Ger-age-res}
	\vskip -0.2in
\end{figure}

\begin{table}
	\centering
	\caption{Pareto optimal solutions - \textit{German credit} - age dataset}
	\begin{tabular}{|c|l|l|l|l|}
		\hline
		\textbf{model}          & \multicolumn{1}{c|}{\textbf{AUC}} & \multicolumn{1}{c|}{\textbf{AOD}} & \multicolumn{1}{c|}{\textbf{ASD}} & \multicolumn{1}{c|}{\textbf{AEOD}} \\ \hline
		\textbf{PR}             & 0.701                             & 0.114                             & 0.002                             & 0.367                              \\ \hline
		\textbf{DI-NN}          & 0.769                             & 0.056                             & 0.013                             & 0.067                              \\ \hline
		\textbf{DI-NN}          & 0.780                             & 0.025                             & 0.054                             & 0.100                              \\ \hline
		\textbf{DI-NN}          & 0.751                             & 0.042                             & 0.043                             & 0.133                              \\ \hline
		\textbf{DI-RF}          & 0.727                             & 0.017                             & 0.011                             & 0.033                              \\ \hline
		\textbf{DI-RF}          & 0.735                             & 0.055                             & 0.100                             & 0.033                              \\ \hline
		\textbf{Reweighing-NN}  & 0.710                             & 0.033                             & 0.005                             & 0.167                              \\ \hline
		\textbf{Reweighing-NN}  & 0.768                             & 0.041                             & 0.109                             & 0.000                              \\ \hline
		\textbf{Reweighing-NN}  & 0.761                             & 0.001                             & 0.087                             & 0.100                              \\ \hline
		\textbf{FAIR-scalar}    & 0.776                             & 0.159                             & 0.037                             & 0.600                              \\ \hline
		\textbf{FAIR-scalar}    & 0.787                             & 0.085                             & 0.103                             & 0.033                              \\ \hline
		\textbf{FAIR-betaSF}    & 0.596                             & 0.000                             & 0.000                             & 0.000                              \\ \hline
		\textbf{FAIR-Bernoulli} & 0.765                             & 0.226                             & 0.002                             & 0.567                              \\ \hline
		\textbf{FAIR-betaREP}   & 0.665                             & 0.003                             & 0.034                             & 0.067                              \\ \hline
		\textbf{FAIR-betaREP}   & 0.780                             & 0.198                             & 0.037                             & 0.400                              \\ \hline
		\textbf{FAD-prob}       & 0.752                             & 0.069                             & 0.070                             & 0.000                              \\ \hline
	\end{tabular}
	\label{tab:Ger-age}
\end{table}

\begin{table}
	\centering
	\caption{Pareto optimal solutions - \textit{German credit} - sex dataset}
	\begin{tabular}{|c|l|l|l|l|}
		\hline
		\textbf{model}         & \multicolumn{1}{c|}{\textbf{AUC}} & \multicolumn{1}{c|}{\textbf{AOD}} & \multicolumn{1}{c|}{\textbf{ASD}} & \multicolumn{1}{c|}{\textbf{AEOD}} \\ \hline
		\textbf{PR}            & 0.540                             & 0.000                             & 0.000                             & 0.000                              \\ \hline
		\textbf{DI-RF}         & 0.727                             & 0.017                             & 0.023                             & 0.033                              \\ \hline
		\textbf{DI-RF}         & 0.725                             & 0.018                             & 0.004                             & 0.017                              \\ \hline
		\textbf{Reweighing-NN} & 0.755                             & 0.002                             & 0.021                             & 0.117                              \\ \hline
		\textbf{Reweighing-RF} & 0.739                             & 0.025                             & 0.014                             & 0.050                              \\ \hline
		\textbf{Reweighing-RF} & 0.739                             & 0.041                             & 0.047                             & 0.033                              \\ \hline
		\textbf{Reweighing-RF} & 0.780                             & 0.018                             & 0.077                             & 0.083                              \\ \hline
		\textbf{Reweighing-RF} & 0.763                             & 0.029                             & 0.072                             & 0.067                              \\ \hline
		\textbf{FAD}           & 0.672                             & 0.013                             & 0.042                             & 0.067                              \\ \hline
		\textbf{FAD}           & 0.766                             & 0.015                             & 0.005                             & 0.167                              \\ \hline
		\textbf{FAD}           & 0.547                             & 0.015                             & 0.019                             & 0.283                              \\ \hline
		\textbf{FAD}           & 0.576                             & 0.014                             & 0.022                             & 0.067                              \\ \hline
		\textbf{FAIR-betaSF}   & 0.804                             & 0.005                             & 0.081                             & 0.000                              \\ \hline
	\end{tabular}
	\label{tab:Ger-sex}%
\end{table}

Consequently, in this section all presented results were evaluated on test sets and only Pareto efficient solutions were presented. Additional results can be observer in~\ref{app:Additional results}.

\subsection{Discussion}

Model behaviour of FAIR model with respect to change of hyperparameter $\alpha$ is shown in Fig.~\ref{fig:FigResultsAlpha} on \textit{Geman credit} dataset. It can be observed that as $\alpha$ decreases, instances which are unfair (but potentially useful for prediction of target variable) are being discarded, so AUC metrics for both the target variable and sensitive attribute decrease. This is experimental verification of theoretical model properties presented in section~\ref{Sec:FAIR}.

\begin{figure}[h]
	\center
	\includegraphics[angle=0, width=0.8\textwidth]{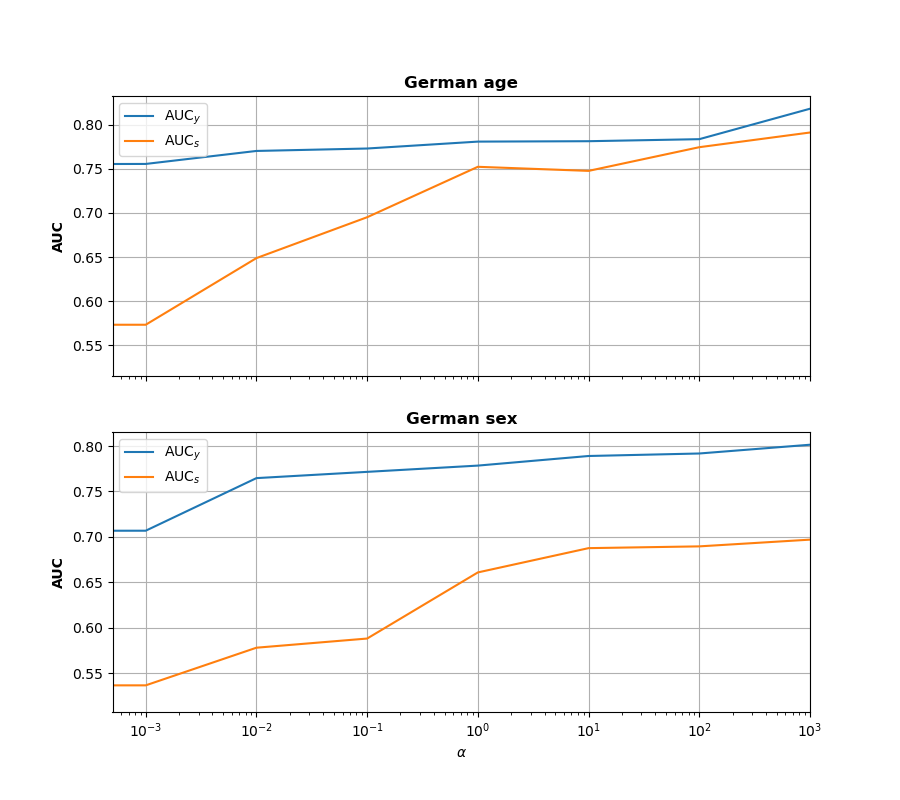}
	\captionsetup{justification=centering}
	\caption{$\mathbf{AUC_y}$ and $\mathbf{AUC_s}$ as functions of the fariness hyperparameter $\alpha$ measured on the  German credit - sex and Readmission datasets ($\mathbf{AUC_y}$ is preferred larger, and $\mathbf{AUC_s}$ smaller)}
	\label{fig:FigResultsAlpha}
\end{figure}

Furthermore, we increased the hyperparameter $\alpha$ in FAIR-scalar model from 0 to the first value where one of the instance in training dataset has weight that tends to 1. Based on theoretical formulation of model properties this is the most "fair" instance in dataset. Moreover, we kept to increase parameter $\alpha$ until the first two instances with weights that tends to 1 in opposite sex and label categories occurred. In table~\ref{tab:Explainability} the attributes of previously mentioned instances are presented.

Firstly, it could be observed that the most "fair" instance has good credit score mainly based on facts that he is employed as manager, does not have other debtors and credits taken, posses life insurance and house, is not a foreign worker, has small amount of money on checking account. Similar, attributes can be seen in the case of the first "fair" instance with good credit score that is female. She is not a foreign worker, employed as manager for 7 or more years, paid back duly existing credits and took credit for buying new car. She has small amount of money on checking account and does not have other debtors. Unlike this two instances, the first instance with bad credit score is unemployed man, that has other debtors, is foreign worker, does have house and car. It is obviously that unemployment and other debts has the most influential impact on labelling this instance as bad.

It can be concluded that all presented instances have reasonable explanations why they are labelled with bad or good credit score. Furthermore, it can be seen that sex does not have any kind of cause on final decision so FAIR-scalar successfully labelled them as "fair".

\begin{table}
	\centering
	\caption{German credit dataset instances with non zero weights}
	\resizebox{\textwidth}{!}{\begin{tabular}{|c|c|c|c|}
		\hline
		\textbf{Credit   duration}                               & 48                         & 36                                                                                   & 36                                                                                   \\ \hline
		\textbf{Credit amount}                                   & 18424                      & 14318                                                                                & 15857                                                                                \\ \hline
		\textbf{Investment as income percentage}                 & 1                          & 4                                                                                    & 2                                                                                    \\ \hline
		\textbf{Residence since}                                 & 2                          & 2                                                                                    & 3                                                                                    \\ \hline
		\textbf{No.of credits taken}                             & 1                          & 1                                                                                    & 1                                                                                    \\ \hline
		\textbf{No. of people liable to provide maintenance for} & 1                          & 1                                                                                    & 1                                                                                    \\ \hline
		\textbf{Status of checking account}                      & \textless{}200 DM          & \textless{}200 DM                                                                    & \textless{}0DM                                                                       \\ \hline
		\textbf{Credit history}                                  & no   credits taken         & existing credits paid back duly till now &  existing credits paid back duly till now \\ \hline
		\textbf{Purpose}                                         & other                      & car (new)                                                                            & other                                                                                \\ \hline
		\textbf{Savings}                                         & \textless{}100 DM          & \textless{}100 DM                                                                    & \textless{}100 DM                                                                    \\ \hline
		\textbf{Employment}                                      & 1\textless{}4 years        & \textgreater{}=7 years                                                               & unemployed                                                                           \\ \hline
		\textbf{Other debtors}                                   & none                       & none                                                                                 & co-applicant                                                                         \\ \hline
		\textbf{Properties}                                      & Life insurance             & unknown                                                                              & car or other                                                                         \\ \hline
		\textbf{Installment plans}                               & bank                       & none                                                                                 & none                                                                                 \\ \hline
		\textbf{Housing}                                         & own                        & for free                                                                             & own                                                                                  \\ \hline
		\textbf{Skill level}                                     & management                 & management                                                                           & self-employed                                                                        \\ \hline
		\textbf{Telphone}                                        & yes,   under customer name & yes,   under customer name                                                           & none                                                                                 \\ \hline
		\textbf{Foreign worker}                                  & no                         & no                                                                                   & yes                                                                                  \\ \hline
		\textbf{Sensitive attribute - Sex}                                             & male                       & female                                                                               & male                                                                                 \\ \hline
		\textbf{Credit score -Label}                                           & Good                       & Good                                                                                 & Bad                                                                                  \\ \hline
	\end{tabular}}
	\label{tab:Explainability}
\end{table}

\section{Conclusions}
\label{Sec:Conclusion}
We introduced a Fair Adversarial Instance Re-weighting (FAIR) discriminative method, which uses adversarial training to learn instance weights to ensure fairness. We proposed four different variants of the method: a non probabilistic one and three models cast in fully probabilistic framework. In addition, we presented a possibility to introduce a baseline to reduce variance of gradient estimation for models based on score function. Theoretical analysis of FAIR model behaviour with respect to the change of the hyperparameter $\alpha$ is given. We proved that changing the value of the hyperparameter controls the trade-off between model fairness and predictive performance. In experimental evaluation on five real-world tasks we demonstrated that our models outperform previous state-of-the-art approaches with respect to fairness metrics and classification performance. Moreover, we showed experimental verification of presented results, and demonstrate that FAIR model is able to find "fair" instances for small values of the hyperparameter $\alpha$.

Further studies should address extending FAIR models to numerical and categorical values of sensitive attributes and adding additional loss constraints for individual fairness.

\section*{Acknowledgement}
This work was supported in part by the ONR/ONR Global under Grant N62909-19-1-2008. In addition, this research is partially supported by the Ministry of Science, Education and Technological Development of the Republic of Serbia grants OI174021, TR35004 and TR41008. The authors would like to express gratitude to company Saga New Frontier Group Belgrade, for supporting this research.

\bibliographystyle{elsarticle-num}
\bibliography{literatura}

\appendix

\section{Model architecture}
\label{app:Architecture}

In all experiments with Reweighing - RF and DI - RF 500 trees were used in the random forest algorithm. Architectures, learning rates and maximum number of epochs used in models with neural networks for all datasets are presented in Table~\ref{tab:A1},~\ref{tab:A2} and~\ref{tab:A3}.

% Table generated by Excel2LaTeX from sheet 'Arhitekture'
\begin{table*}
	\centering
	\caption{Architectures of models used}
	\label{Table:tab1}
	\begin{adjustbox}{max width=\textwidth, totalheight = \textheight-0.1in}
	\begin{tabular}{|L|L|L|L|L|L|}
			\toprule
			\textbf{Model} & {\textbf{No. of units per layer $P_\theta(w|\mathbf{x})$ or $P_\theta(\mathbf{z}|\mathbf{x})$}} & {\textbf{No. of units per layer $P_\phi(\mathbf{y}|\mathbf{x})$}} & {\textbf{No. of cells per layer $P_\psi(\mathbf{y}|\mathbf{x})$}} & \textbf{Activation} & {\textbf{Early stopping epoch / learning rate}} \\
			
			\midrule
			\multicolumn{6}{|c|}{\textbf{Adult}} \\
			\midrule
			\textbf{DI - NN} & - & 62/41/27/1 & - & ReLU + Batch normalization  + sigmoid (last layer) & $10 / 10^{-3}$ \\
			\midrule
			\textbf{Reweighing - NN} & - & 62/41/27/1 & - & ReLU + Batch normalization  + sigmoid (last layer) & $10 / 10^{-3}$ \\
			\midrule
			\textbf{FAD} & 62/41/27 & 18/12/1 & 18/12/1 & ReLU + Batch normalization  + sigmoid (last layer) & $50 / 10^{-4}$ \\
			\midrule
			\textbf{FAD-prob} & {46/23/23/23} & 11/1 & 11/1 & ReLU   + sigmoid (last layer) & $50 / 10^{-4}$ \\
			\midrule
			\textbf{FAIR-scalar} & 62/41/27/1 & 62/41/1 & 62/1 & ReLU + Batch normalization  + sigmoid (last layer) & $50 / 10^{-4}$ \\
			\midrule
			\textbf{FAIR-betaSF} & 62/41/27/2 & 62/41/1 & 62/1 & ReLU + Batch normalization  + sigmoid or exp (last layer) & $50 / 10^{-4}$ \\
			\midrule
			\textbf{FAIR-betaREP} & 62/41/27/2 & 62/41 & 62/1 & ReLU + Batch normalization  + sigmoid or exp (last layer) & $50 / 10^{-5}$ \\
			\midrule
			\textbf{FAIR-Bernoulli} & 62/41/1 & 62/41/27/1 & 62/1 & ReLU + Batch normalization  + sigmoid  (last layer) & $500 / 10^{-4}$ \\
			\midrule
			\multicolumn{6}{|c|}{\textbf{Readmission}} \\
			\midrule
			\textbf{DI - NN} & - & 464/232/116/1 & - & ReLU + Batch normalization  + sigmoid (last layer) & $10 / 10^{-3}$ \\
			\midrule
			\textbf{Reweighing - NN} & - & 464/232/116/1 & - & ReLU + Batch normalization  + sigmoid (last layer) & $10 / 10^{-3}$ \\
			\midrule
			\textbf{FAD} & 464/232/116 & 58/1 & 58/1 & ReLU + Batch normalization  + sigmoid (last layer) & $40 / 10^{-4}$ \\
			\midrule
			\textbf{FAD-prob} & {464/232/232/232} & 116/58/1 & 116/58/1 & ReLU   + sigmoid (last layer) & $40 / 10^{-4}$ \\
			\midrule
			\textbf{FAIR-scalar} & 464/232/1 & 464/1 & 464/1 & ReLU + Batch normalization  + sigmoid (last layer) & $40 / 10^{-4}$ \\
			\midrule
			\textbf{FAIR-betaSF} & 464/232/2 & 464/1 & 464/1 & ReLU + Batch normalization  + sigmoid or exp (last layer) & $40 / 10^{-5}$ \\
			\midrule
			\textbf{FAIR-betaREP} & 464/232/2 & 464/1 & 464/1 & ReLU + Batch normalization  + sigmoid or exp (last layer) & $40 / 10^{-5}$ \\
			\midrule
			\textbf{FAIR-Bernoulli} & 464/1 & 464/232/1 & 464/1 & ReLU + Batch normalization  + sigmoid (last layer) & $40 / 10^{-4}$ \\
			\bottomrule
	\end{tabular}%
	\end{adjustbox}
	\label{tab:A1}%
\end{table*}%

\begin{table*}
	\centering
	\caption{Architectures of models used}
	\label{Table:tab1}
	\begin{adjustbox}{max width=\textwidth, totalheight = \textheight-0.1in}
		\begin{tabular}{|L|L|L|L|L|L|}
			\toprule
			\textbf{Model} & {\textbf{No. of units per layer $P_\theta(w|\mathbf{x})$ or $P_\theta(\mathbf{z}|\mathbf{x})$} } & {\textbf{No. of units per layer $P_\phi(\mathbf{y}|\mathbf{x})$}} & {\textbf{No. of cells per layer $P_\psi(\mathbf{y}|\mathbf{x})$}} & \textbf{Activation} & {\textbf{Early stopping epoch / learning rate}} \\
			
			\midrule
			\multicolumn{6}{|c|}{\textbf{Medical expenditures}} \\
			\midrule
			\textbf{DI - NN} & - & 91/60/1 & - & ReLU + Batch normalization  + sigmoid (last layer) & $10 / 10^{-3}$ \\
			\midrule
			\textbf{Reweighing - NN} & - & 91/60/1 & - & ReLU + Batch normalization  + sigmoid (last layer) & $10 / 10^{-3}$ \\
			\midrule
			\textbf{FAD} & 68/34/17 & 8/1 & 8/1 & ReLU + Batch normalization  + sigmoid (last layer) & $50 / 10^{-4}$ \\
			\midrule
			\textbf{FAD-prob} & {68/34/34/34} & 17/1 & 17/1 & ReLU   + sigmoid (last layer) & $50 / 10^{-4}$ \\
			\midrule
			\textbf{FAIR-scalar} & 68/34/1 & 68/1 & 68/1 & ReLU + Batch normalization  + sigmoid (last layer) & $50 / 10^{-4}$ \\
			\midrule
			\textbf{FAIR-betaSF} & 68/34/2 & 68/1 & 68/1 & ReLU + Batch normalization  + sigmoid or exp (last layer) & $50 / 10^{-5}$ \\
			\midrule
			\textbf{FAIR-betaREP} & 68/34/2 & 68/1 & 68/1 & ReLU + Batch normalization  + sigmoid (last layer) & $50 / 10^{-5}$ \\
			\midrule
			\textbf{FAIR-Bernoulli} & 68/1 & 68/34/1 & 68/1 & ReLU + Batch normalization  + sigmoid or exp (last layer) & $50 / 10^{-4}$ \\
			\midrule
			\multicolumn{6}{|c|}{\textbf{German credit - sex}} \\
			\midrule
			\textbf{DI - NN} & - & 37/24/1 & - & ReLU + Batch normalization  + sigmoid (last layer) & $10 / 10^{-3}$ \\
			\midrule
			\textbf{Reweighing - NN} & - & 37/24/1 & - & ReLU + Batch normalization  + sigmoid (last layer) & $10 / 10^{-3}$ \\
			\midrule
			\textbf{FAD} & 37/24/1 & 16/1 & 16/1 & ReLU + Batch normalization  + sigmoid (last layer) & $60 / 10^{-4}$ \\
			\midrule
			\textbf{FAD-prob} & {28/14/14/14} & 7/1 & 7/1 & ReLU   + sigmoid (last layer) & $50 / 10^{-5}$ \\
			\midrule
			\textbf{FAIR-scalar} & 37/1 & 1  & 1  & ReLU + Batch normalization  + sigmoid (last layer) & $50 / 10^{-4}$ \\
			\midrule
			\textbf{FAIR-betaSF} & 37/2 & 1  & 1  & ReLU + Batch normalization  + sigmoid or exp (last layer) & $60 / 10^{-5}$ \\
			\midrule
			\textbf{FAIR-betaREP} & 37/2 & 1  & 1  & ReLU + Batch normalization  + sigmoid or exp (last layer) & $60 / 10^{-5}$ \\
			\midrule
			\textbf{FAIR-Bernoulli} & 37/1 & 1  & 1  & ReLU + Batch normalization  + sigmoid (last layer) & $60 / 10^{-4}$ \\
			\bottomrule
		\end{tabular}%
	\end{adjustbox}
	\label{tab:A2}%
\end{table*}%

\begin{table*}
	\centering
	\caption{Architectures of models used}
	\label{Table:tab3}
	\begin{adjustbox}{max width=\textwidth}
		\begin{tabular}{|L|L|L|L|L|L|}
			\toprule
			\textbf{Model} & {\textbf{No. of units per layer $P_\theta(w|\mathbf{x})$ or $P_\theta(\mathbf{z}|\mathbf{x})$} } & {\textbf{No. of units per layer $P_\phi(\mathbf{y}|\mathbf{x})$}} & {\textbf{No. of cells per layer $P_\psi(\mathbf{y}|\mathbf{x})$}} & \textbf{Activation} & {\textbf{Early stopping epoch / learning rate}} \\
			\midrule
			\multicolumn{6}{|c|}{\textbf{German credit - age}} \\
			\midrule
			\textbf{DI - NN} & - & 37/24/1 & - & ReLU + Batch normalization  + sigmoid (last layer) & $10 / 10^{-3}$ \\
			\midrule
			\textbf{Reweighing - NN} & - & 37/24/1 & - & ReLU + Batch normalization  + sigmoid (last layer) & $10 / 10^{-3}$ \\
			\midrule
			\textbf{FAD} & 37/24/1 & 16/1 & 16/1 & ReLU + Batch normalization  + sigmoid (last layer) & $50 / 10^{-4}$ \\
			\midrule
			\textbf{FAD-prob} & {28/14/14/14} & 7/1 & 7/1 & ReLU   + sigmoid (last layer) & $50 / 10^{-5}$ \\
			\midrule
			\textbf{FAIR-scalar} & 37/1 & 37/1 & 1  & ReLU + Batch normalization  + sigmoid (last layer) & $50 / 10^{-4}$ \\
			\midrule
			\textbf{FAIR-betaSF} & 37/2 & 37/1 & 1  & ReLU + Batch normalization  + sigmoid or exp (last layer) & $50 / 10^{-5}$ \\
			\midrule
			\textbf{FAIR-betaREP} & 37/2 & 37/1 & 1  & ReLU + Batch normalization  + sigmoid or exp (last layer) & $50 / 10^{-5}$ \\
			\midrule
			\textbf{FAIR-Bernoulli} & 37/1 & 37/1 & 1  & ReLU + Batch normalization  + sigmoid (last layer) & $50 / 10^{-4}$ \\
			\bottomrule
		\end{tabular}%
	\end{adjustbox}
	\label{tab:A3}%
\end{table*}%

\section{Additional results}
\label{app:Additional results}

More detailed results of experimental evaluation are given in Figs.~\ref{fig:Adult all}, ~\ref{fig:Readmission all},~\ref{fig:MEPS19 all},~\ref{fig:Ger_age all} and~\ref{fig:Ger_sex all} by presenting Pareto fronts with all dominated and non-dominated models.

\begin{figure*}
	\center
	\includegraphics[angle=0, width=1\textwidth]{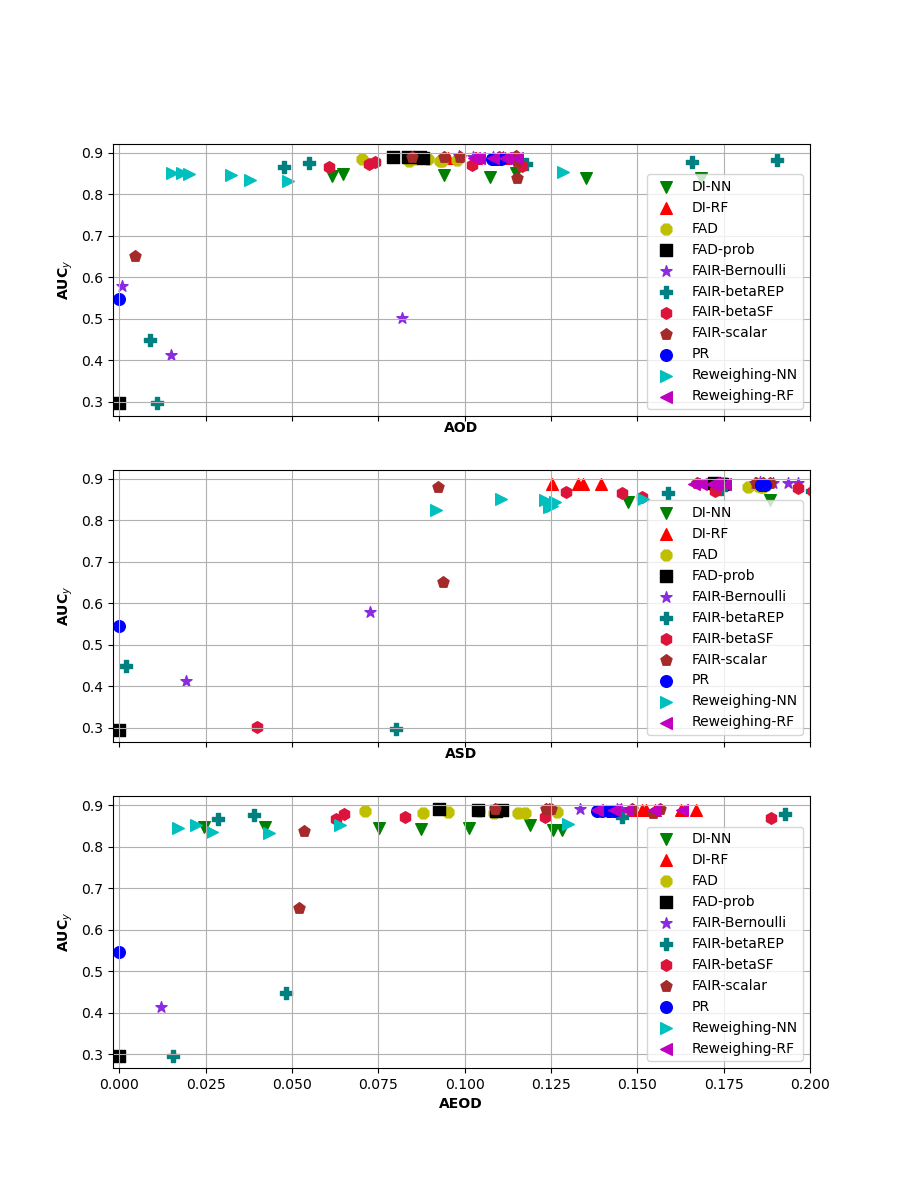}
	\captionsetup{justification=centering}
	\caption{Classification performance and fairness of models as measured by $\mathbf{AUC_y}$ and $\mathbf{AOD}$, $\mathbf{ASD}$ or $\mathbf{AEOD}$ on the \textit{Adult income} dataset}
	\label{fig:Adult all}
	\vskip -0.2in
\end{figure*}

\begin{figure*}
	\center
	\includegraphics[angle=0, width=1\textwidth]{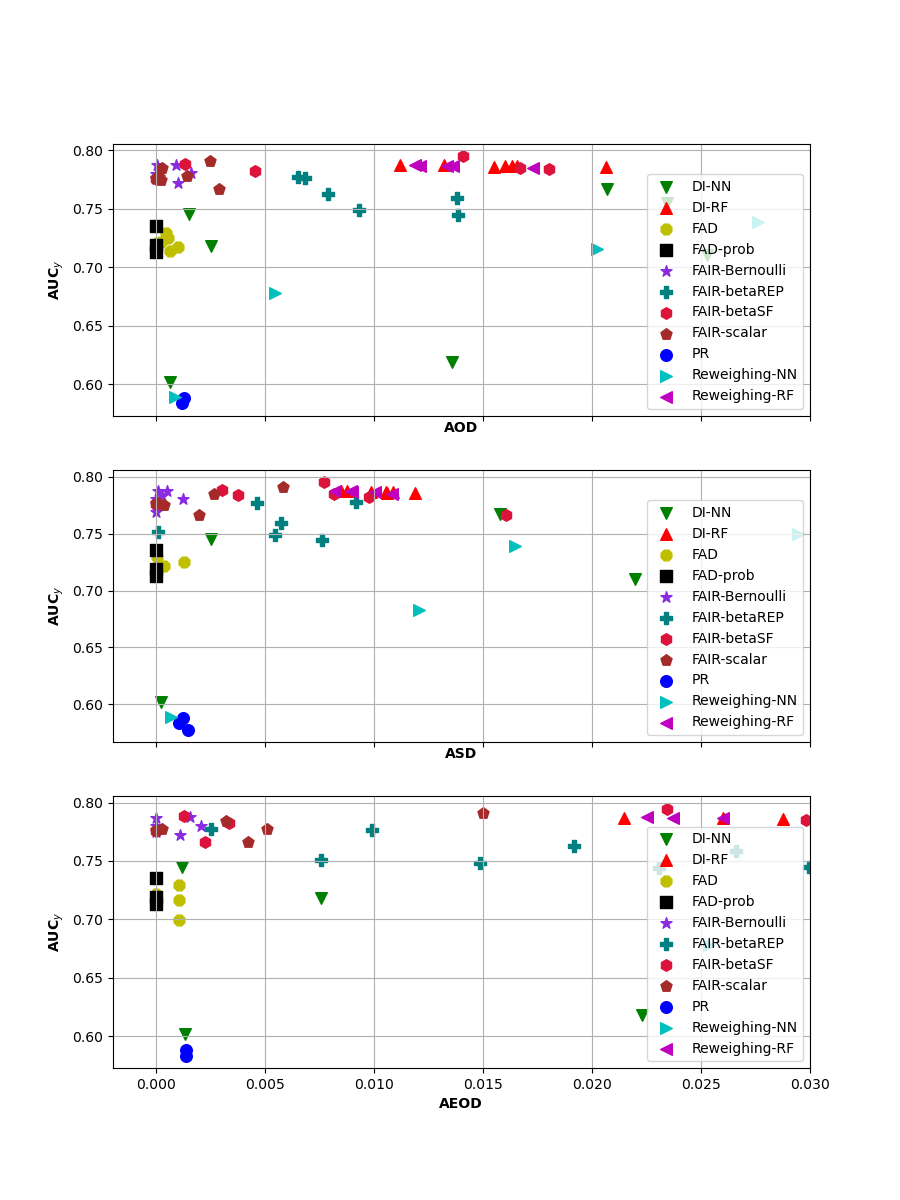}
	\captionsetup{justification=centering}
	\caption{Classification performance and fairness of models as measured by $\mathbf{AUC_y}$ and $\mathbf{AOD}$, $\mathbf{ASD}$ or $\mathbf{AEOD}$ on the \textit{Hospital readmission} dataset}
	\label{fig:Readmission all}
	\vskip -0.2in
\end{figure*}

\begin{figure*}
	\center
	\includegraphics[angle=0, width=1\textwidth]{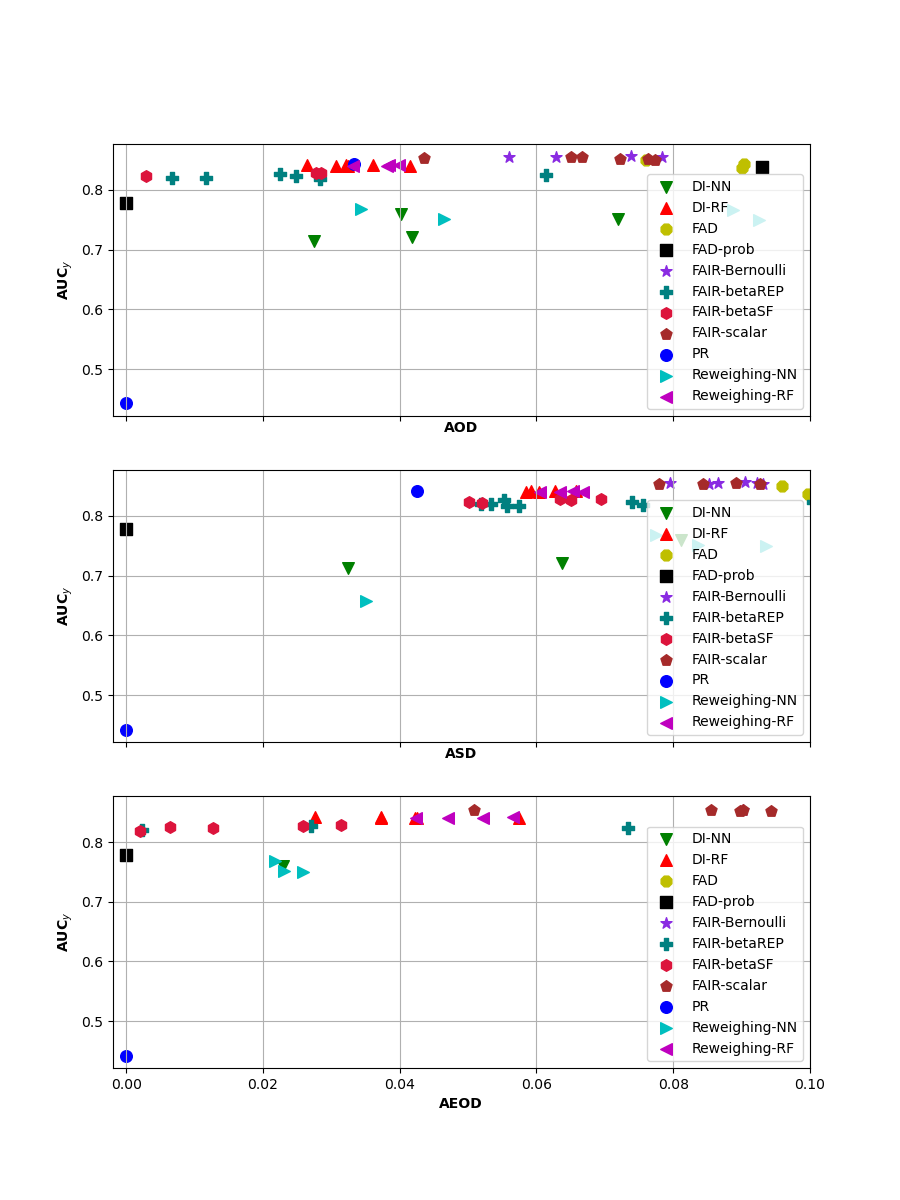}
	\captionsetup{justification=centering}
	\caption{Classification performance and fairness of models as measured by $\mathbf{AUC_y}$ and $\mathbf{AOD}$, $\mathbf{ASD}$ or $\mathbf{AEOD}$ on the \textit{Hospital expenditures} dataset}
	\label{fig:MEPS19 all}
	\vskip -0.2in
\end{figure*}

\begin{figure*}
	\center
	\includegraphics[angle=0, width=1\textwidth]{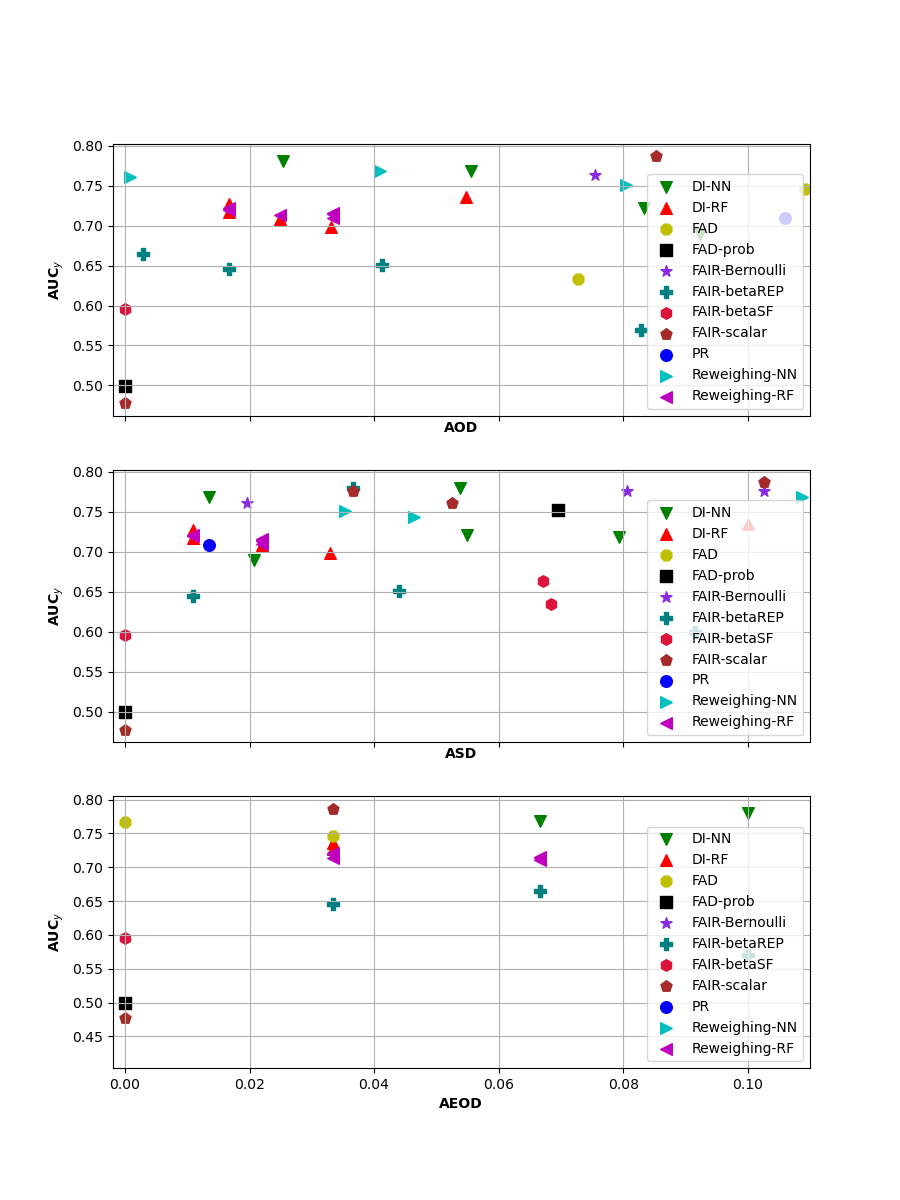}
	\captionsetup{justification=centering}
	\caption{Classification performance and fairness of models as measured by $\mathbf{AUC_y}$ and $\mathbf{AOD}$, $\mathbf{ASD}$ or $\mathbf{AEOD}$ on the \textit{German credit} (age) dataset}
	\label{fig:Ger_age all}
	\vskip -0.2in
\end{figure*}

\begin{figure*}
	\center
	\includegraphics[angle=0, width=1\textwidth]{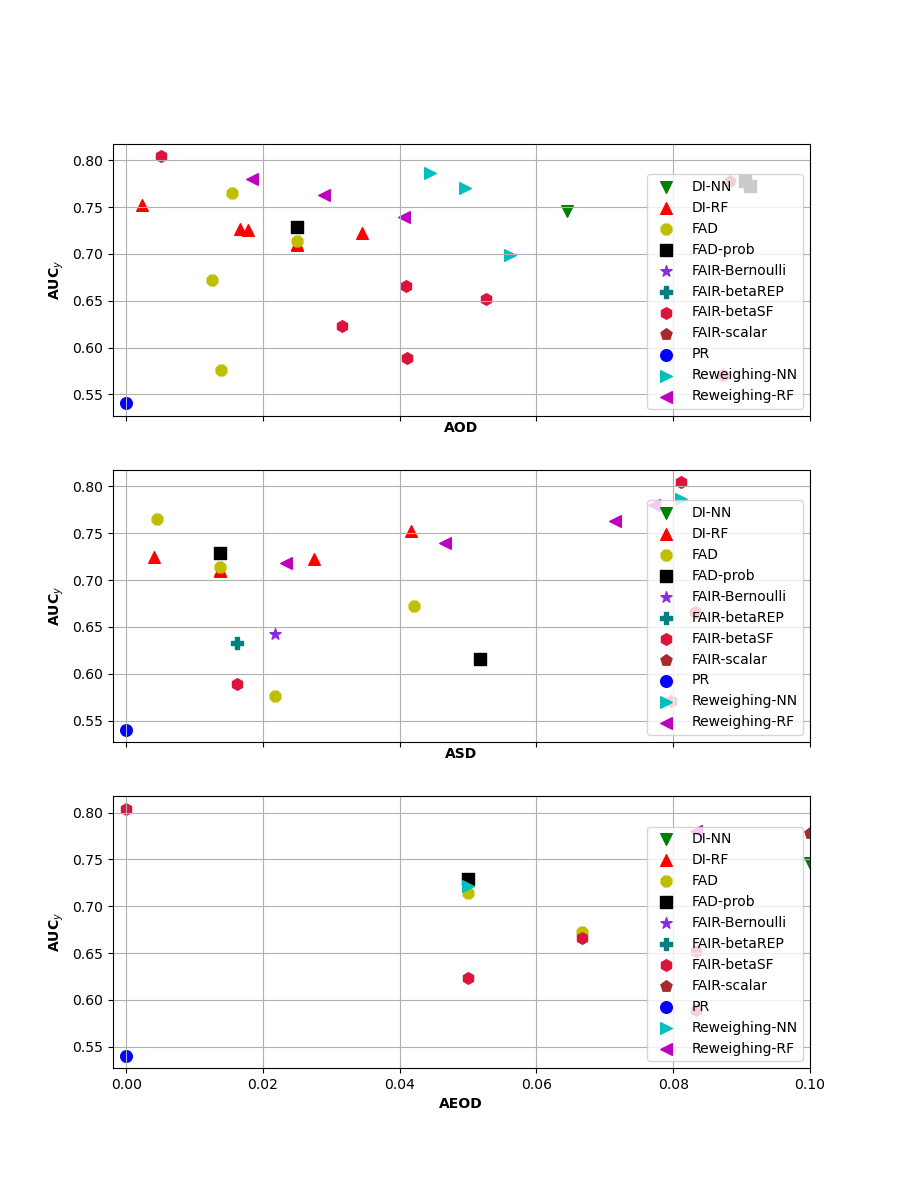}
	\captionsetup{justification=centering}
	\caption{Classification performance and fairness of models as measured by $\mathbf{AUC_y}$ and $\mathbf{AOD}$, $\mathbf{ASD}$ or $\mathbf{AEOD}$ on the \textit{German credit} (sex) dataset}
	\label{fig:Ger_sex all}
	\vskip -0.2in
\end{figure*}

\end{document}